\documentclass[twoside,10pt]{article}

\usepackage{jmlr2e}

\jmlrheading{1}{2000}{1-48}{4/00}{10/00}{Chuanhao Li and Hongning Wang}

\ShortHeadings{Distributed Kernelized Contextual Bandits}{Li and Wang$\times 3$}
\firstpageno{1}

\usepackage[utf8]{inputenc} 
\usepackage[T1]{fontenc}    
\usepackage{hyperref}       
\usepackage{url}            
\usepackage{booktabs}       
\usepackage{amsfonts}       
\usepackage{nicefrac}       
\usepackage{microtype}      
\usepackage{graphicx}
\usepackage[noend]{algorithmic}
\usepackage{algorithm}
\usepackage{wrapfig}

\usepackage{graphicx}
\usepackage{subfigure}
\usepackage{amsmath}
\usepackage{bbm}

\DeclareMathOperator*{\argmax}{arg\,max}
\DeclareMathOperator*{\argmin}{arg\,min}
\def \cU {\mathcal{U}}
\def \cD {\mathcal{D}}
\def \cN {\mathcal{N}}
\def \bU {\mathbf{U}}

\def \cH {\mathcal{H}}
\def \cF {\mathcal{F}}
\def \bA {\mathbf{A}}
\def \bx {\mathbf{x}}
\def \bX {\mathbf{X}}

\def \by {\mathbf{y}}

\def \cA {\mathcal{A}}

\def \bb {\mathbf{b}}

\def \bI {\mathbf{I}}
\def \bR {\mathbb{R}}
\def \bE {\mathbb{E}}
\def \bbE {\mathbf{E}}

\def \cB {\mathcal{B}}
\def \cS {\mathcal{S}}
\def \bS {\mathbf{S}}
\def \bK {\mathbf{K}}

\def \bZ {\mathbf{Z}}
\def \bz {\mathbf{z}}
\def \bP {\mathbf{P}}
\def \bPhi {\mathbf{\Phi}}
\def \btheta {\mathbf{\theta}}

\newcommand{\algone}{{DisKernelUCB}}
\newcommand{\algtwo}{{Approx-DisKernelUCB}}

\begin{document}

\title{Communication Efficient Distributed Learning for \\Kernelized Contextual Bandits}

\author{\name Chuanhao Li \email cl5ev@virginia.edu \\
       \addr Department of Computer Science\\
       University of Virginia\\
      Charlottesville, VA 22903, USA\\
        \AND
       \name Huazheng Wang \email huazheng.wang@oregonstate.edu \\
       \addr School of Electrical Engineering and Computer Science\\
       Oregon State University\\
      Princeton, NJ 08544, USA\\
        \AND
       \name Mengdi Wang \email mengdiw@princeton.edu \\
       \addr Department of Electrical and Computer Engineering\\
       Princeton University\\
      Princeton, NJ 08544, USA\\
       \AND
       \name Hongning Wang \email hw5x@virginia.edu \\
       \addr Department of Computer Science\\
       University of Virginia\\
      Charlottesville, VA 22903, USA
       }


\maketitle

\begin{abstract}
We tackle the communication efficiency challenge of learning kernelized contextual bandits in a distributed setting. Despite the recent advances in communication-efficient distributed bandit learning, existing solutions are restricted to simple models like multi-armed bandits and linear bandits, which hamper their practical utility. 
In this paper, instead of assuming the existence of a linear reward mapping from the features to the expected rewards, we consider non-linear reward mappings, by letting agents collaboratively search in a reproducing kernel Hilbert space (RKHS). This introduces significant challenges in communication efficiency as distributed kernel learning requires the transfer of raw data, leading to a communication cost that grows linearly w.r.t. time horizon $T$. We addresses this issue by equipping all agents to communicate via a common Nystr\"{o}m embedding that gets updated adaptively as more data points are collected. We rigorously proved that our algorithm can attain sub-linear rate in both regret and communication cost.
\end{abstract}


\section{Introduction} \label{sec:intro}
Contextual bandit algorithms have been widely used for a variety of real-world applications, including recommender systems \citep{li2010contextual}, display advertisement \citep{li2010exploitation} and clinical trials \citep{durand2018contextual}. While most existing bandit solutions assume a centralized setting (i.e., all the data reside in and all the actions are taken by a central server), 
there is 
increasing research effort on distributed bandit learning lately \citep{wang2019distributed,dubey2020differentially,shi2021federated,huang2021federated,li2022asynchronous},
where $N$ clients, under the coordination of a central server, collaborate 
to minimize the overall cumulative regret incurred over a finite time horizon $T$. In many distributed application scenarios, communication is the main bottleneck, e.g., communication in a network of mobile devices can be slower than local computation by several orders of magnitude \citep{huang2013depth}.
Therefore, it is vital for distributed bandit learning algorithms to attain sub-linear rate (w.r.t. time horizon $T$) in both cumulative regret and communication cost.

However, prior works in this line of research are restricted to linear models \citep{wang2019distributed}, which could oversimplify the problem and thus leads to inferior performance in practice. In centralized setting, kernelized bandit algorithms, e.g., KernelUCB \citep{valko2013finite} and IGP-UCB \citep{chowdhury2017kernelized}, are proposed to address this issue by modeling the unknown reward mapping as a non-parametric function lying in a reproducing kernel Hilbert space (RKHS), i.e., the expected reward is linear w.r.t. an action feature map of possibly infinite dimensions.
Despite the strong modeling capability of kernel method, collaborative exploration in the RKHS gives rise to additional challenges in designing a communication efficient bandit algorithm.
Specifically, unlike distributed linear bandit where the clients can simply communicate the $d \times d$ sufficient statistics \citep{wang2019distributed}, where $d$ is the dimension of the input feature vector, the \emph{joint kernelized estimation} of the unknown reward function requires communicating either 1) the $p \times p$ sufficient statistics in the RKHS, where $p$ is the dimension of the RKHS that is possibly infinite, or 2) the set of input feature vectors that grows linearly w.r.t. $T$. Neither of them is practical due to the huge communication cost. 

In this paper, we propose the first communication efficient algorithm for distributed kernel bandits, which tackles the aforementioned challenge via a low-rank approximation of the empirical kernel matrix. 
In particular, we extended the Nystr\"{o}m method \citep{nystrom1930praktische} to distributed learning for kernelized contextual bandits. In this solution, all clients first project their local data to a finite RKHS spanned by a common dictionary, i.e., a small subset of the original dataset, and then they only need to communicate the embedded statistics for collaborative exploration. To ensure effective regret reduction after each communication round, as well as ensuring the dictionary remains representative for the entire distributed dataset throughout the learning process, the frequency of dictionary update and synchronization of embedded statistics is controlled by measuring the amount of new information each client has gained since last communication.
We rigorously prove that the proposed algorithm incurs an $O(N^{2} \gamma_{NT}^{3})$ communication cost, where $\gamma_{NT}$ is the maximum information gain that is known to be $O\bigl(\log (NT) \bigr)$ for kernels with exponentially decaying eigenvalues, which includes the most commonly used Gaussian kernel, while attaining the optimal $O(\sqrt{NT}\gamma_{NT})$ cumulative regret.
\section{Related Works} \label{sec:related_works}
To balance exploration and exploitation in stochastic linear contextual bandits, LinUCB algorithm \citep{li2010contextual,abbasi2011improved} is commonly used, which selects arm optimistically w.r.t. a constructed confidence set on the unknown linear reward function.
By using kernels and Gaussian processes, studies in \cite{srinivas2009gaussian,valko2013finite,chowdhury2017kernelized} further extend UCB algorithms to non-parametric reward functions in RKHS, i.e., the feature map associated with each arm is possibly infinite.

Recent years have witnessed increasing research efforts in distributed bandit learning, i.e., multiple agents collaborate in pure exploration \cite{hillel2013distributed,tao2019collaborative,du2021collaborative}, or regret minimization \cite{shi2021federated,wang2019distributed,li2022asynchronous}.
They mainly differ in the relations of learning problems solved by the agents (i.e., homogeneous vs., heterogeneous) and the type of communication network (i.e., peer-to-peer (P2P) vs., star-shaped).
Most of these works assume linear reward functions, and the clients communicate by transferring the $O(d^{2})$ sufficient statistics.
Korda et al. \cite{korda2016distributed} considered a peer-to-peer (P2P) communication network and assumed that the clients form clusters, i.e., each cluster is associated with a unique bandit problem. 
Huang et al. \cite{huang2021federated} considered a star-shaped communication network as in our paper, but their proposed phase-based elimination algorithm only works in the fixed arm set setting.
The closest works to ours are \citep{wang2019distributed,dubey2020differentially,li2022asynchronous}, which proposed event-triggered communication protocols to obtain sub-linear communication cost over time for distributed linear bandits with a time-varying arm set.
In comparison, distributed kernelized contextual bandits still remain under-explored. The only existing work in this direction \citep{dubey2020kernel} considered heterogeneous agents, where each agent is associated with an additional feature describing the task similarity between agents. However, they assumed a local communication setting, where the agent immediately shares the new raw data point to its neighbors after each interaction, and thus the communication cost is still linear over time. 

Another closely related line of works is kernelized bandits with approximation, where Nystr\"{o}m method is adopted to improve computation efficiency in a centralized setting. Calandriello et al. \cite{calandriello2019gaussian} proposed an algorithm named BKB, which uses Ridge Leverage Score sampling (RLS) to re-sample a new dictionary from the updated dataset after each interaction with the environment. A recent work by Zenati et al. \cite{zenati2022efficient} further improved the computation efficiency of BKB by adopting an online sampling method to update the dictionary. However, both of them updated the dictionary at each time step to ensure the dictionary remains representative w.r.t. the growing dataset, and therefore are not applicable to our problem. This is because the dataset is stored cross clients in a distributed manner, and projecting the dataset to the space spanned by the new dictionary requires communication with all clients, which is prohibitively expensive in terms of communication. Calandriello et al. \cite{calandriello2020near} also proposed a variant of BKB, named BBKB, for batched Gaussian process optimization. BBKB only needs to update the dictionary occasionally according to an adaptive schedule, and thus partially addresses the issue mentioned above. 
However, as BBKB works in a centralized setting, their adaptive schedule can be computed based on the whole batch of data, while in our decentralized setting, each client can only make the update decision according to the data that is locally available. Moreover, in BBKB, all the interactions are based on a fixed model estimation over the whole batch, which is mentioned in their Appendix A.4 as a result of an inherent technical difficulty. In comparison, our proposed method effectively addresses this difficulty with improved analysis, and thus allows each client to utilize newly collected data to update its model estimation on the fly.

\section{Preliminaries}
In this section, we first formulate the problem of distributed kernelized contextual bandits.
Then, as a starting point,
we propose and analyze a naive UCB-type algorithm for distributed kernelized contextual bandit problem, named \algone{}. This demonstrates the challenges in designing a communication efficient algorithm for this problem, and also lays down the foundation for further improvement on communication efficiency in Section \ref{sec:method}.

\subsection{Distributed Kernelized Contextual Bandit Problem} \label{subsec:problem_formulation}
Consider a learning system with 1) $N$ clients that are responsible for taking actions and receiving feedback from the environment, and 2) a central server that coordinates the communication among the clients. The clients cannot directly communicate with each other, but only with the central server, i.e., a star-shaped communication network.
Following prior works \citep{wang2019distributed,dubey2020differentially}, we assume the $N$ clients interact with the environment in a round-robin manner for a total number of $T$ rounds. 

Specifically, at round $l \in [T]$, each client $i \in [N]$ chooses an arm $\bx_{t}$ from a candidate set $\cA_{t}$, and then receives the corresponding reward feedback $y_{t}=f(\bx_{t}) + \eta_{t} \in \bR$, where the subscript $t:=N (l-1) + i$ indicates this is the $t$-th interaction between the learning system and the environment, and we refer to it as time step $t$ \footnote{The meaning of index $t$ is slightly different from prior works, e.g. DisLinUCB in \citep{wang2019distributed}, but this is only to simplify the use of notation and does not affect the theoretical results}.
Note that $\cA_{t}$ is a time-varying subset of $\cA \subseteq \bR^{d}$ that is possibly infinite, $f$ denotes the unknown reward function shared by all the clients, and $\eta_{t}$ denotes the noise.

Denote the sequence of indices corresponding to the interactions between client $i$ and the environment up to time $t$ as $\cN_{t}(i)=\left\{1 \leq s \leq t: i_{s}=i \right\}$ (if $s \;\mathrm{mod}\; N =0$, then $i_{s}=N$; otherwise $i_{s}=s \;\mathrm{mod}\; N$) for $t=1,2,\dots,NT$.
By definition, $|\cN_{Nl}(i)|=l,\forall l\in[T]$, i.e., the clients have equal number of interactions at the end of each round $l$.

\paragraph{Kernelized Reward Function}
We consider an unknown reward function $f$ that lies in a RKHS, denoted as $\cH$, such that the reward can be equivalently written as $$y_{t} = \btheta_{\star}^{\top} \phi(\bx_{t}) + \eta_{t},$$ where $\btheta_{\star} \in \cH$ is an unknown parameter, and $\phi: \bR^{d} \rightarrow \cH$ is a known feature map associated with $\cH$. 
We assume $\eta_{t}$ is zero-mean $R$-sub-Gaussian conditioned on $\sigma\bigl( (\bx_{s},\eta_{s})_{s \in \cN_{t-1}(i_{t})}, \bx_{t} \bigr), \forall t$, which denotes the $\sigma$-algebra generated by client $i_{t}$'s previously pulled arms and the corresponding noise.
In addition, there exists a positive definite kernel $k(\cdot,\cdot)$ associated with $\cH$, 
and we assume $\forall \bx \in \cA$ that, $\lVert \bx \rVert_{k} \leq L$ and $\lVert f \rVert_{k} \leq S$ for some $L,S >0$.

\paragraph{Regret and Communication Cost}
The goal of the learning system is to minimize the cumulative (pseudo) regret for all $N$ clients, i.e., $R_{NT}=\sum_{t=1}^{NT} r_{t}$, where $r_{t}=\max_{\bx \in \cA_{t}} \phi(\bx)^{\top} \btheta_{\star}- \phi(\bx_{t})^{\top} \btheta_{\star}$.
Meanwhile, the learning system also wants to keep the communication cost $C_{NT}$ low, which is measured by 
the total number of scalars being transferred across the system up to time step $NT$.

\subsection{Distributed Kernel UCB} \label{subsec:diskernel_ucb}
As a starting point to studying the communication efficient algorithm in Section \ref{sec:method} and demonstrate the challenges in designing a communication efficient distributed kernelized contextual bandit algorithm, here we first introduce and analyze a naive algorithm where the $N$ clients collaborate on learning the exact parameters of kernel bandit, i.e., the mean and variance of estimated reward.
We name this algorithm Distributed Kernel UCB, or \algone{} for short, and its description is given in Algorithm \ref{alg:diskernelucb}.

\begin{algorithm}[t] 
    \caption{Distributed Kernel UCB (\algone{})} 
    \label{alg:diskernelucb}
  \begin{algorithmic}[1]
    \STATE \textbf{Input} threshold $D>0$
    \STATE \textbf{Initialize} $t_\text{last}=0$, $\cD_{0}(i)=\Delta \cD_{0}(i)=\emptyset, \forall i \in [N]$
    \FOR{ round $l=1,2,...,T$}
        \FOR{ client $i = 1,2,...,N$}
            \STATE Client $i$ chooses arm $\bx_{t} \in \cA_{t}$ according to Eq~\eqref{eq:UCB_exact} and observes reward $y_{t}$, where $t=N(l-1)+i$
            \STATE Client $i$ updates 
            $\bK_{\cD_{t}(i),\cD_{t}(i)}, \by_{\cD_{t}(i)}$, where $\cD_{t}(i)=\cD_{t-1}(i) \cup \{t\}$; and its upload buffer $\Delta \cD_{t}(i)=\Delta \cD_{t-1}(i) \cup \{t\}$\\
            \textit{// Global Synchronization}
            \IF{the event $\cU_{t}(D)$ defined in Eq~\eqref{eq:sync_event_exact} is true}
                \STATE \textbf{Clients} $\forall j \in [N]$: send $\{(\bx_{s},y_{s})\}_{s \in \Delta \cD_{t}(j)}$ to server, and reset $\Delta \cD_{t}(j)=\emptyset$
                \STATE \textbf{Server}: aggregates and sends back $\{(\bx_{s},y_{s})\}_{s \in [t]}$; sets $t_{\text{last}}=t$
                \STATE \textbf{Clients} $\forall j \in [N]$: update $\bK_{\cD_{t}(j),\cD_{t}(j)}, \by_{\cD_{t}(i)}$, where $\cD_{t}(j)=[t]$ 
            \ENDIF
        \ENDFOR
    \ENDFOR
  \end{algorithmic}
\end{algorithm}

\paragraph{Arm Selection} 
For each round $l \in [T]$, when client $i \in [N]$ interacts with the environment, i.e., the $t$-th interaction between the learning system and the environment where $t=N(l-1)+i$, 
it chooses arm $\bx_{t} \in \cA_{t}$ based on the UCB of the mean estimator (line 5):
\begin{equation} \label{eq:UCB_exact}
    \bx_{t}=\argmax_{\bx \in \cA_{t}} \hat{\mu}_{t-1,i}(\bx) + \alpha_{t-1,i} \hat{\sigma}_{t-1,i}(\bx)
\end{equation}
where $\hat{\mu}_{t,i}(\bx)$ and $\hat{\sigma}^{2}_{t,i}(\bx)$ denote client $i$'s local estimated mean reward for arm $\bx \in \cA$ and its variance, and $\alpha_{t-1,i}$ is a carefully chosen scaling factor to balance exploration and exploitation (see Lemma \ref{lem:regret_comm_diskernelucb} for proper choice).

To facilitate further discussion, for time step $t \in [NT]$, we denote the sequence of time indices for the data points that have been used to update client $i$'s local estimate as $\cD_{t}(i)$, which include both data points collected locally and those shared by the other clients. If the clients never communicate, $\cD_{t}(i)=\cN_{t}(i),\forall t,i$; otherwise, $\cN_{t}(i) \subset \cD_{t}(i) \subseteq [t]$, with $\cD_{t}(i)=[t]$ recovering the centralized setting, i.e., each new data point collected from the environment immediately becomes available to all the clients in the learning system. 
The design matrix and reward vector for client $i$ at time step $t$ are denoted by $\bX_{\cD_{t}(i)}=[\bx_{s}]_{s \in \cD_{t}(i)}^{\top} \in \bR^{|\cD_{t}(i)| \times d}, \by_{t,i}=[y_{s}]_{s \in \cD_{t}(i)}^{\top} \in \bR^{|\cD_{t}(i)|}$, respectively. By applying the feature map $\phi(\cdot)$ to each row of $\bX_{\cD_{t}(i)}$, we obtain $\bPhi_{\cD_{t}(i)} \in \bR^{|\cD_{t}(i)| \times p}$, where $p$ is the dimension of $\cH$ and is possibly infinite.
Since the reward function is linear in $\cH$, client $i$ can construct the Ridge regression estimator $\hat{\theta}_{t,i} = (\bPhi_{\cD_{t}(i)}^{\top} \bPhi_{\cD_{t}(i)} + \lambda I)^{-1} \bPhi_{\cD_{t}(i)}^{\top} \by_{t,i}$,
where $\lambda > 0$ is the regularization coefficient. This gives us the estimated mean reward and variance in primal form for any arm $\bx \in \cA$, i.e., $\hat{\mu}_{t,i}(\bx) = \phi(\bx)^{\top} \bA_{t,i}^{-1} \bb_{t,i}$ and $\hat{\sigma}_{t,i}(\bx) = \sqrt{\phi(\bx)^{\top} \bA_{t,i}^{-1} \phi(\bx) }$,
where $\bA_{t,i}=\bPhi_{\cD_{t}(i)}^{\top} \bPhi_{\cD_{t}(i)}+ \lambda \bI$ and $\bb_{t,i}=\bPhi_{\cD_{t}(i)}^{\top} \by_{t,i}$.
Then using the kernel trick, we can obtain their equivalence in the dual form that only involves entries of the kernel matrix, and avoids directly working on $\cH$ which is possibly infinite:
\begin{align*}
    \hat{\mu}_{t,i}(\bx) &= \bK_{\cD_{t}(i)}(\bx)^{\top} \bigl( \bK_{\cD_{t}(i), \cD_{t}(i)} + \lambda I \bigr)^{-1} \by_{\cD_{t}(i)}  \\
    \hat{\sigma}_{t,i}(\bx) &= \lambda^{-1/2}\sqrt{ k(\bx, \bx) - \bK_{\cD_{t}(i)}(\bx)^{\top} \bigl( \bK_{\cD_{t}(i), \cD_{t}(i)} + \lambda I \bigr)^{-1} \bK_{\cD_{t}(i)}(\bx) }
\end{align*}
where $\bK_{\cD_{t}(i)}(\bx) = \bPhi_{\cD_{t}(i)} \phi(\bx) = [k(\bx_{s}, \bx)]^{\top}_{s \in \cD_{t}(i)} \in \bR^{|\cD_{t}(i)|}$, and $\bK_{\cD_{t}(i), \cD_{t}(i)} = \bPhi_{\cD_{t}(i)}^{\top} \bPhi_{\cD_{t}(i)} = [k(\bx_{s}, \bx_{s^{\prime}})]_{s, s^{\prime} \in \cD_{t}(i)} \in \bR^{|\cD_{t}(i)| \times |\cD_{t}(i)|}$.


\paragraph{Communication Protocol}
To reduce the regret in future interactions with the environment, the $N$ clients need to collaborate via communication, and a carefully designed communication protocol is essential in ensuring the communication efficiency.
In prior works like DisLinUCB \cite{wang2019distributed}, after each round of interaction with the environment, client $i$ checks whether the event $\{(|\cD_{t}(i)|-|\cD_{t_\text{last}}(i)|) \log(\frac{\det(\bA_{t,i})}{\det(\bA_{t_\text{last},i})})>D\}$ is true, where $t_\text{last}$ denotes the time step of last global synchronization. If true, a new global synchronization is triggered, such that the server will require all clients to upload their sufficient statistics since $t_\text{last}$, aggregate them to compute $\{\bA_{t},\bb_{t}\}$, and then synchronize the aggregated sufficient statistics with all clients, i.e., set $\{\bA_{t,i},\bb_{t,i}\}=\{\bA_{t},\bb_{t}\},\forall i \in [N]$.

Using kernel trick, we can obtain an equivalent event-trigger in terms of the kernel matrix,
\begin{equation} \label{eq:sync_event_exact}
    \cU_{t}(D) = \left\{(|\cD_{t}(i_{t})|-|\cD_{t_\text{last}}(i_{t})|) \log\left(\frac{\det(\bI + \lambda^{-1} \bK_{\cD_{t}(i_{t}), \cD_{t}(i_{t})} )}{\det(\bI + \lambda^{-1} \bK_{\cD_{t}(i_{t})\setminus \Delta \cD_{t}(i_{t}), \cD_{t}(i_{t}) \setminus \Delta \cD_{t}(i_{t})})}\right) > D \right\}.
\end{equation}
where $D>0$ denotes the predefined threshold value.
If event $\cU_{t}(D)$ is true (line 7), a global synchronization is triggered (line 7-10), where the local datasets of all $N$ clients are synchronized to $\{(\bx_{s},y_{s})\}_{s \in [t]}$. We should note that the transfer of raw data $(\bx_{s},y_{s})$ is necessary for the update of the kernel matrix and reward vector in line 6 and line 10, which will be used for arm selection at line 5.
This is an inherent disadvantage of kernelized estimation in distributed settings, which, as we mentioned in Section \ref{sec:related_works}, is also true for the existing distributed kernelized bandit algorithm \cite{dubey2020kernel}. Lemma \ref{lem:regret_comm_diskernelucb} below shows that in order to obtain the optimal order of regret, \algone{} incurs a communication cost linear in $T$ (proof given in the appendix), which is expensive for an online learning problem.
\begin{lemma}[Regret and Communication Cost of \algone{}] \label{lem:regret_comm_diskernelucb}
With threshold $D=\frac{T}{N \gamma_{NT}}$, 
$\alpha_{t,i}=\sqrt{\lambda} \lVert \btheta_{\star} \rVert + R \sqrt{4 \ln{N/\delta}+ 2\ln{\det( \bI + \bK_{\cD_{t}(i), \cD_{t}(i)} /\lambda)}}$, we have 
\begin{align*}
    R_{NT} = O \bigl( \sqrt{NT}(\lVert \theta_{\star} \rVert \sqrt{\gamma_{NT}} + \gamma_{NT} ) \bigr),
\end{align*}
with probability at least $1-\delta$, and
\begin{align*}
    C_{NT} = O(T N^{2} d ).
\end{align*}
where $\gamma_{NT}:=\max_{\cD \subset \cA:|\cD|=NT} \frac{1}{2}\log \det(\bK_{\cD,\cD}/\lambda + \bI)$ is the maximum information gain after $NT$ interactions \citep{chowdhury2017kernelized}. It is problem-dependent and can be bounded for specific arm set $\cA$ and kernel function $k(\cdot, \cdot)$. For example, $\gamma_{NT} = O(d \log(NT))$ for linear kernel and $\gamma_{NT} = O(\log(NT)^{d+1})$ for Gaussian kernel. 
\end{lemma}
\begin{remark} \label{rmk:1}
In the distributed linear bandit problem, to attain $O(d\sqrt{NT} \ln(NT))$ regret, DisLinUCB \citep{wang2019distributed} requires a total number of $O(N^{0.5} d \log(NT))$ synchronizations, and \algone{} matches this result under linear kernel, as it requires $O(N^{0.5} \gamma_{NT})$ synchronizations. 
We should note that the communication cost for each synchronization in DisLinUCB is fixed, i.e., $O(Nd^{2})$ to synchronize the sufficient statistics with all the clients, so in total $C_{NT}=O(N^{1.5} d^{3} \ln(NT) )$. However, this is not the case for \algone{} that needs to send raw data, because the communication cost for each synchronization in \algone{} is not fixed, but depends on the number of unshared data points on each client. Even if the total number of synchronizations is small, \algone{} could still incur $C_{NT}=O(TN^{2}d)$ in the worse case. Consider the extreme case where synchronization only happens once, but it happens near $NT$, then we still have $C_{NT}=O(T N^{2} d )$. The time when synchronization gets triggered depends on $\{\cA_{t} \}_{t \in [NT]}$, which is out of the control of the algorithm.
Therefore, in the following section, to improve the communication efficiency of \algone{}, we propose to let each client communicate embedded statistics in some small subspace during each global synchronization. 
\end{remark}

\section{Approximated Distributed Kernel UCB} \label{sec:method}
In this section, we propose and analyze a new algorithm that improves the communication efficiency of \algone{} using the Nystr\"{o}m approximation, such that the clients only communicate the embedded statistics during event-triggered synchronizations.
We name this algorithm Approximated Distributed Kernel UCB, or \algtwo{} for short. Its description is given in Algorithm \ref{algo:Sync-KernelUCB-Approx}.

\begin{algorithm}[t]
    \caption{Approximated Distributed Kernel UCB (\algtwo{})} 
    \label{algo:Sync-KernelUCB-Approx}
  \begin{algorithmic}[1]
    \STATE \textbf{Input:} threshold $D>0$, regularization parameter $\lambda>0$, $\delta \in (0,1)$ and kernel function $k(\cdot, \cdot)$.
    \STATE \textbf{Initialize} $\tilde{\mu}_{0,i}(\bx)= 0, \tilde{\sigma}_{0,i}(\bx)=\sqrt{k(\bx,\bx)}$, $\cN_{0}(i)=\cD_{0}(i)=\emptyset$, $\forall i\in[N]$; $\cS_{0}=\emptyset$, $t_{\text{last}}=0$
    \FOR{ round $l=1,2,...,T$}
        \FOR{ client $i = 1,2,...,N$}
            \STATE [Client $i$] selects arm $\bx_{t} \in \cA_{t}$ according to Eq~\eqref{eq:UCB_approx}
            and observes reward $y_{t}$, where $t:=N(l-1)+i$
            \STATE [Client $i$] 
            updates $\bZ_{\cD_{t}(i);\cS_{t_{\text{last}}}}^{\top}\bZ_{\cD_{t}(i);\cS_{t_{\text{last}}}}$ and $\bZ_{\cD_{t}(i);\cS_{t_{\text{last}}}}^{\top} \by_{\cD_{t}(i)}$ using $\bigl(\bz(\bx_{t};\cS_{t_\text{last}}),y_{t}\bigr)$;
            sets $\cN_{t}(i)=\cN_{t-1}(i)\cup \{t\}$, and $\cD_{t}(i)=\cD_{t-1}(i)\cup \{t\}$ \\
            \textit{// Global Synchronization}
            \IF{the event $\cU_{t}(D)$ defined in Eq~\eqref{eq:sync_event} is true}
                \STATE [Clients $\forall i $] sample $\cS_{t,i}=\text{RLS}(\cN_{t}(i),\bar{q},\tilde{\sigma}_{t_\text{last},i}^{2})$, and send $\{(\bx_{s},y_{s})\}_{s \in \cS_{t,i}}$ to server
                \STATE [Server] aggregates and sends $\{(\bx_{s},y_{s})\}_{s \in \cS_{t}}$ back to all clients, where $\cS_{t}=\cup_{i \in [N]} \cS_{t,i}$
                \STATE [Clients $\forall i$] compute and send $\{\bZ_{\cN_{t}(i);\cS_{t}}^{\top}\bZ_{\cN_{t}(i);\cS_{t}}, \bZ_{\cN_{t}(i);\cS_{t}}^{\top}\by_{\cN_{t}(i)}\}$ to server
                \STATE [Server] aggregates 
                $\sum_{i=1}^{N}\bZ_{\cN_{t}(i);\cS_{t}}^{\top}\bZ_{\cN_{t}(i);\cS_{t}}, \sum_{i=1}^{N}\bZ_{\cN_{t}(i);\cS_{t}}^{\top}\by_{\cN_{t}(i)}$
                and sends it back
                \STATE [Clients $\forall i$] 
                updates $\bZ_{\cD_{t}(i);\cS_{t}}^{\top}\bZ_{\cD_{t}(i);\cS_{t}}$ and $\bZ_{\cD_{t}(i);\cS_{t}}^{\top} \by_{\cD_{t}(i)}$; sets $\cD_{t}(i)=\cup_{i=1}^{N}\cN_{t}(i)=[t]$ and $t_\text{last}=t$
            \ENDIF
        \ENDFOR
    \ENDFOR
  \end{algorithmic}
\end{algorithm}

\subsection{Algorithm}
\paragraph{Arm selection}
For each round $l \in [T]$, when client $i \in [N]$ interacts with the environment, i.e., the $t$-th interaction between the learning system and the environment where $t:=N(l-1)+i$, 
instead of using the UCB for the exact estimator in Eq~\eqref{eq:UCB_exact}, client $i$ chooses arm $\bx_{t} \in \cA_{t}$ that maximizes the UCB for the approximated estimator (line 5):
\begin{equation} \label{eq:UCB_approx}
    \bx_{t}=\argmax_{\bx \in \cA_{t,i}} \tilde{\mu}_{t-1,i}(\bx) + \alpha_{t-1,i} \tilde{\sigma}_{t-1,i}(\bx)
\end{equation}
where $\tilde{\mu}_{t-1,i}(\bx)$ and $\tilde{\sigma}_{t-1,i}(\bx)$ are approximated using Nyestr\"{o}m method, and the statistics used to compute these approximations are much more efficient to communicate as they scale with the maximum information gain $\gamma_{NT}$ instead of $T$.

Specifically, Nystr\"{o}m method works by projecting some original dataset $\cD$ to the subspace defined by a small representative subset $\cS \subseteq \cD$, which is called the dictionary. The orthogonal projection matrix is defined as
\begin{align*}
    \bP_{\cS} = \bPhi_{\cS}^{\top} \bigl( \bPhi_{\cS} \bPhi_{\cS}^{\top} \bigr)^{-1} \bPhi_{\cS}=\bPhi_{\cS}^{\top} \bK_{\cS,\cS}^{-1} \bPhi_{\cS} \in \bR^{p \times p}
\end{align*}
We then take eigen-decomposition of $\bK_{\cS,\cS} = \bU \mathbf{\Lambda} \bU^{\top}$ to rewrite the orthogonal projection as $\bP_{\cS}=\bPhi_{\cS}^{\top} \bU \mathbf{\Lambda}^{-1/2} \mathbf{\Lambda}^{-1/2} \bU^{\top} \bPhi_{\cS}$, and define the Nystr\"{o}m embedding function
\begin{align*}
    z(\bx;\cS) = \bP_{\cS}^{1/2} \phi(\bx) =\mathbf{\Lambda}^{-1/2} \bU^{\top} \bPhi_{\cS} \phi(\bx) = \bK_{\cS,\cS}^{-1/2} \bK_{\cS}(\bx)
\end{align*}
which maps the data point $\bx$ from $\bR^{d}$ to $\bR^{|\cS|}$.

Therefore, we can approximate the Ridge regression estimator in Section \ref{subsec:diskernel_ucb} as $\tilde{\theta}_{t,i} = \tilde{\bA}_{t,i}^{-1} \tilde{\bb}_{t,i}$,
where $\tilde{\bA}_{t,i}=\bP_{\cS} \bPhi_{\cD_{t}(i)}^{\top}\bPhi_{\cD_{t}(i)} \bP_{\cS}+ \lambda \bI $, and $\tilde{\bb}_{t,i}=\bP_{\cS}\bPhi_{\cD_{t}(i)}^{\top} \by_{\cD_{t}(i)}$, and thus the approximated mean reward and variance in Eq~\eqref{eq:UCB_approx} can be expressed as $\tilde{\mu}_{t,i}(\bx) = \phi(\bx)^{\top} \tilde{\bA}_{t,i}^{-1} \tilde{\bb}_{t,i}$ and $\tilde{\sigma}_{t,i}(\bx) = \sqrt{ \phi(\bx)^{\top} \tilde{\bA}_{t,i}^{-1} \phi(\bx)}$,
and their kernelized representation are (see appendix for detailed derivation)
\begin{align*}
    \tilde{\mu}_{t,i}(\bx) & = z(\bx;\cS)^{\top} \bigl( \bZ_{\cD_{t}(i);\cS}^{\top}\bZ_{\cD_{t}(i);\cS} + \lambda \bI\bigr)^{-1} \bZ_{\cD_{t}(i);\cS}^{\top} \by_{\cD_{t}(i)} \\
    \tilde{\sigma}_{t,i}(\bx) & = \lambda^{-1/2}\sqrt{ k(\bx, \bx) -  z(\bx;\cS)^{\top} \bZ_{\cD_{t}(i);\cS}^{\top}\bZ_{\cD_{t}(i);\cS} [ \bZ_{\cD_{t}(i);\cS}^{\top}\bZ_{\cD_{t}(i);\cS} + \lambda \bI]^{-1} z(\bx|\cS)  } 
\end{align*}
where $\bZ_{\cD_{t}(i);\cS} \in \bR^{|\cD_{t}(i)|\times |\cS|}$ is obtained by applying $z(\cdot;\cS)$ to each row of $\bX_{\cD_{t}(i)}$, i.e., $\bZ_{\cD_{t}(i);\cS}= \bPhi_{\cD_{t}(i)} \bP_{\cS}^{1/2}$.
We can see that 
the computation of $\tilde{\mu}_{t,i}(\bx)$ and $\tilde{\sigma}_{t,i}(\bx)$ only requires the embedded statistics: matrix $\bZ_{\cD_{t}(i);\cS}^{\top}\bZ_{\cD_{t}(i);\cS} \in \bR^{|\cS| \times |\cS|}$ and vector $\bZ_{\cD_{t}(i);\cS}^{\top} \by_{\cD_{t}(i)} \in \bR^{|\cS|}$, which, as we will show later, makes joint kernelized estimation among $N$ clients much more efficient in communication.


After obtaining the new data point $(\bx_{t}, y_{t})$, client $i$ immediately updates both $\tilde{\mu}_{t-1,i}(\bx)$ and $\tilde{\sigma}_{t-1,i}(\bx)$ using the newly collected data point $(\bx_{t},y_{t})$, i.e., by projecting $\bx_{t}$ to the finite dimensional RKHS spanned by $\bPhi_{\cS_{t_\text{last}}}$ (line 6).
Recall that, we use $\cN_{t}(i)$ to denote the sequence of indices for data collected by client $i$, and denote by $\cD_{t}(i)$ the sequence of indices for data that has been used to update client $i$'s model estimation $\tilde{\mu}_{t,i}$. Therefore, both of them need to be updated to include time step $t$.


\paragraph{Communication Protocol}
With the approximated estimator, the size of message being communicated across the learning system is reduced. However, a carefully designed event-trigger is still required to minimize the total number of global synchronizations up to time $NT$. Since the clients can no longer evaluate the exact kernel matrices in Eq~\eqref{eq:sync_event_exact}, we instead use the event-trigger in Eq~\eqref{eq:sync_event}, which can be computed using the approximated variance from last global synchronization as,
\begin{equation} \label{eq:sync_event}
    \cU_{t}(D) = \left\{\sum_{s \in \cD_{t}(i) \setminus \cD_{t_\text{last}}(i)} \tilde{\sigma}_{t_\text{last},i}^{2}(\bx_{s}) > D\right\}
\end{equation}
Similar to Algorithm \ref{alg:diskernelucb}, if Eq \eqref{eq:sync_event} is true, global synchronization is triggered, where both the dictionary and the embedded statistics get updated. 
During synchronization, each client first samples a subset $\cS_{t}(i)$ from $\cN_{t}(i)$ (line 8) using Ridge Leverage Score sampling (RLS) \citep{calandriello2019gaussian,calandriello2020near}, which is given in Algorithm \ref{alg:rls}, and then sends $\{(\bx_{s},y_{s})\}_{s \in \cS_{t}(i)}$ to the server.
The server aggregates the received local subsets to construct a new dictionary $\{(\bx_{s},y_{s})\}_{s \in \cS_{t}}$, where $\cS_{t}=\cup_{i=1}^{N}\cS_{t}(i)$, and then sends it back to all $N$ clients (line 9). 
Finally, the $N$ clients use this updated dictionary to re-compute the embedded statistics of their local data, and then synchronize it with all other clients via the server (line 10-12).

\begin{algorithm}[h]
    \caption{$\quad \text{Ridge Leverage Score Sampling (RLS)}$} \label{alg:rls}
  \begin{algorithmic}[1]
    \STATE \textbf{Input:} dataset $\cD$, scaling factor $\bar{q}$, (possibly delayed and approximated) variance function $\tilde{\sigma}^{2}(\cdot)$
    \STATE \textbf{Initialize} new dictionary $\cS=\emptyset$
    \FOR{$s \in \cD$}
        \STATE Set $\tilde{p}_{s}=\bar{q}\tilde{\sigma}^{2}(\bx_{s})$
        \STATE Draw $q_{s} \sim \text{Bernoulli}(\tilde{p}_{s})$
        \STATE If $q_{s}=1$, add $s$ into $\cS$ 
    \ENDFOR
    \STATE \textbf{Output:} $\cS$
  \end{algorithmic}
\end{algorithm}
Intuitively, in Algorithm \ref{algo:Sync-KernelUCB-Approx}, the clients first agree upon a common dictionary $\cS_{t}$ that serves as a good representation of the whole dataset at the current time $t$, and then project their local data to the subspace spanned by this dictionary before communication, in order to avoid directly sending the raw data as in Algorithm \ref{alg:diskernelucb}. 
Then using the event-trigger, each client monitors the amount of new knowledge it has gained through interactions with the environment from last synchronization. When there is a sufficient amount of new knowledge, it will inform all the other clients to perform a synchronization.
As we will show in the following section, the size of $\cS_{t}$ scales linearly w.r.t. the maximum information gain $\gamma_{NT}$, and therefore it improves both the local computation efficiency on each client, and the communication efficiency during the global synchronization.

\subsection{Theoretical Analysis} \label{subsec:theoretical_analysis}
Denote the sequence of time steps when global synchronization is performed, i.e., the event $\cU_{t}(D)$ in Eq~\eqref{eq:sync_event} is true, as $\{t_{p}\}_{p=1}^{B}$, where $B \in [NT]$ denotes the total number of global synchronizations.
Note that in Algorithm \ref{algo:Sync-KernelUCB-Approx}, the dictionary is only updated during global synchronization, e.g., at time $t_{p}$, the dictionary $\{(\bx_{s}, y_{s})\}_{s \in \cS_{t_{p}}}$ is sampled from the whole dataset $\{(\bx_{s}, y_{s})\}_{s \in [t_{p}]}$ in a distributed manner, and remains fixed for all the interactions happened at $t \in [t_{p}+1,t_{p+1}]$.
Moreover, at time $t_{p}$, all the clients synchronize their embedded statistics, so that $\cD_{t_{p}}(i)=[t_{p}],\forall i \in [N]$.

Since Algorithm \ref{algo:Sync-KernelUCB-Approx} enables local update on each client, for time step $t \in [t_{p}+1,t_{p}]$, new data points are collected and added into $\cD_{t}(i)$, such that $\cD_{t}(i) \supseteq [t_{p}]$. This \textit{decreases} the approximation accuracy of $\cS_{t_{p}}$, as new data points may not be well approximated by $\cS_{t_{p}}$. For example, in extreme cases, the new data could be orthogonal to the dictionary. To formally analyze the accuracy of the dictionary, we adopt the definition of $\epsilon$-accuracy from \cite{calandriello2017second}. Denote by $\bar{\bS}_{t,i} \in \bR^{|\cD_{t}(i)| \times |\cD_{t}(i)|}$ a diagonal matrix, with its $s$-th diagonal entry equal to $\frac{1}{\sqrt{\tilde{p}_{s}}}$ if $s \in \cS_{t_{p}}$ and $0$ otherwise. Then if
\begin{align*}
    (1-\epsilon_{t,i}) (\bPhi_{\cD_{t}(i)}^{\top} \bPhi_{\cD_{t}(i)} + \lambda \bI) \preceq \bPhi_{\cD_{t}(i)}^{\top} \bar{\bS}_{t,i}^{\top} \bar{\bS}_{t,i}\bPhi_{\cD_{t}(i)} + \lambda \bI \preceq (1+\epsilon_{t,i}) (\bPhi_{\cD_{t}(i)}^{\top} \bPhi_{\cD_{t}(i)} + \lambda \bI),
\end{align*}
we say the dictionary $\{(\bx_{s},y_{s})\}_{s \in \cS_{t_{p}}}$ is $\epsilon_{t,i}$-accurate w.r.t. dataset $\{(\bx_{s},y_{s})\}_{s \in \cD_{t}(i)}$. 

As shown below, the accuracy of the dictionary for Nystr\"{o}m approximation is essential as it affects the width of the confidence ellipsoid, and thus affects the cumulative regret. Intuitively, in order to ensure its accuracy throughout the learning process, we need to 1) make sure the RLS procedure in line 8 of Algorithm \ref{algo:Sync-KernelUCB-Approx} that happens at each global synchronization produces a representative set of data samples, and 2) monitor the extent to which the dictionary obtained in previous global synchronization has degraded over time, and when necessary, trigger a new global synchronization to update it. Compared with prior work that freezes the model in-between consecutive communications \cite{calandriello2020near}, the analysis of $\epsilon$-accuracy for \algtwo{} is unique to our paper and the result is presented below.
\begin{lemma} \label{lem:dictionary_accuracy_global}
With $\bar{q}=6\frac{1+\epsilon}{1-\epsilon} \log(4NT/\delta)/\epsilon^{2}$, for some $\epsilon \in [0,1)$, and threshold $D>0$, Algorithm \ref{algo:Sync-KernelUCB-Approx} guarantees that the dictionary is accurate with constant $\epsilon_{t,i}:=\bigl(\epsilon+1- \frac{1}{1+\frac{1+\epsilon}{1-\epsilon}D} \bigr)$, and its size $|\cS_{t}| = O(\gamma_{NT})$ for all $t \in [NT]$. 
\end{lemma}
Based on Lemma \ref{lem:dictionary_accuracy_global}, we can construct the following confidence ellipsoid for unknown parameter $\theta_{\star}$.

\begin{lemma}[Confidence Ellipsoid of Approximated Estimator] \label{lem:confidence_ellipsoid_approx}
Under the condition that $\bar{q}=6\frac{1+\epsilon}{1-\epsilon} \log(4NT/\delta)/\epsilon^{2}$, for some $\epsilon \in [0,1)$, and threshold $D>0$, with probability at least $1-\delta$, we have $\forall t,i$ that
\small
\begin{align*}
    \lVert\tilde{\btheta}_{t,i} - \btheta_{\star} \rVert_{\tilde{\bA}_{t,i}} 
    & \leq \Bigl( \frac{1}{\sqrt{-\epsilon + 1/(\frac{1+\epsilon}{1-\epsilon}D)}}  + 1  \Bigr) \sqrt{\lambda} \lVert \btheta_{\star} \rVert + 2R \sqrt{\ln{N/\delta}+ \gamma_{NT}} := \alpha_{t,i}.
\end{align*}
\end{lemma}

Using Lemma \ref{lem:confidence_ellipsoid_approx}, we obtain the regret and communication cost upper bound of \algtwo{}, which is given in Theorem \ref{thm:regret_comm_sync} below.
\begin{theorem}[Regret and Communication Cost of \algtwo{}] \label{thm:regret_comm_sync}
Under the same condition as Lemma \ref{lem:confidence_ellipsoid_approx}, and by setting $D=\frac{1}{N}, \epsilon < \frac{1}{3}$, we have
\begin{align*}
    R_{NT} = O \bigl( \sqrt{NT}(\lVert \theta_{\star} \rVert \sqrt{\gamma_{NT}} + \gamma_{NT} ) \bigr)
\end{align*}
with probability at least $1-\delta$, and 
\begin{align*}
    C_{NT} = O\bigl( N^{2} \gamma_{NT}^{3} \bigr)
\end{align*}
\end{theorem}
Here we provide a proof sketch for Theorem \ref{thm:regret_comm_sync}, and the complete proof can be found in appendix.
\begin{proof}[Proof Sketch]
Similar to the analysis of \algone{} in Section \ref{subsec:diskernel_ucb} and DisLinUCB from \citep{wang2019distributed}, the cumulative regret incurred by \algtwo{} can be decomposed in terms of `good' and `bad' epochs, and bounded separately. Here an epoch refers to the time period in-between two consecutive global synchronizations, e.g., the $p$-th epoch refers to $[t_{p-1}+1,t_{p}]$. 
Now consider an imaginary centralized agent that has immediate access to each data point in the learning system, and denote by $A_{t}=\sum_{s=1}^{t} \phi_{s} \phi_{s}^{\top}$ for $t \in [NT]$ the matrix constructed by this centralized agent.
We call the $p$-th epoch a good epoch if $\ln(\frac{\det(\bI + \lambda^{-1}\bK_{[t_{p}],[t_{p}]})}{\det(\bI + \lambda^{-1}\bK_{[t_{p-1}],[t_{p-1}]})}) \leq 1$,
otherwise it is a bad epoch. 
Note that 
$\ln(\frac{\det(\bI + \lambda^{-1}\bK_{[t_{1}],[t_{1}]})}{\det(\bI)})+\ln(\frac{\det(\bI + \lambda^{-1}\bK_{[t_{2}],[t_{2}]})}{\det(\bI + \lambda^{-1}\bK_{[t_{1}],[t_{1}]})})+\dots+\ln(\frac{\det(\bI + \lambda^{-1}\bK_{[NT],[NT]})}{\det(\bI + \lambda^{-1}\bK_{[t_{B}],[t_{B}]})}) = \ln(\det(\bI + \lambda^{-1} \bK_{[NT],[NT]})) \leq 2 \gamma_{NT}$, where the last equality is due to the matrix determinant lemma, and the last inequality is by the definition of the maximum information gain $\gamma_{NT}$ in Lemma \ref{lem:regret_comm_diskernelucb}.
Then based on the pigeonhole principle, there can be at most $2 \gamma_{NT}$ bad epochs.

By combining Lemma \ref{lem:rls} and Lemma \ref{lem:confidence_ellipsoid_approx}, 
we can bound the cumulative regret incurred during all good epochs, i.e., $R_{good} = O(\sqrt{NT}\gamma_{NT})$, which matches the optimal regret attained by the KernelUCB algorithm in centralized setting. 
Our analysis deviates from that of \algone{} in the bad epochs, because of the difference in the event-trigger.
Previously, the event-trigger of \algone{} directly bounds the cumulative regret each client incurs during a bad epoch, i.e., 
\small
$\sum_{t \in \cD_{t_{p}}(i) \setminus \cD_{t_{p-1}}(i) }\hat{\sigma}_{t-1,i}(\bx_{t}) \leq \sqrt{ (|\cD_{t_{p}}(i)|-|\cD_{t_{p-1}}(i)|) \log\left(\det(\bI + \lambda^{-1} \bK_{\cD_{t}(i_{t}), \cD_{t}(i_{t})} )/\det(\bI + \lambda^{-1} \bK_{\cD_{t}(i_{t})\setminus \Delta \cD_{t}(i_{t}), \cD_{t}(i_{t}) \setminus \Delta \cD_{t}(i_{t})})\right)} < \sqrt{D}$. 
\normalsize
However, the event trigger of \algtwo{} only bounds part of it, i.e., $\sum_{t \in \cD_{t_{p}}(i) \setminus \cD_{t_{p-1}}(i) }\tilde{\sigma}_{t-1,i}(\bx_{t}) \leq  \sqrt{ (|\cD_{t_{p}}(i)|-|\cD_{t_{p-1}}(i)|) D } $, which leads to $R_{bad}=O(\sqrt{T}\gamma_{NT} N \sqrt{D})$ that is slightly worse than that of \algone{}, i.e., a $\sqrt{T}$ factor in place of the $\sqrt{\gamma_{NT}}$ factor.
By setting $D=1/N$, we have $R_{NT}=O(\sqrt{NT}\gamma_{NT})$. Note that, to make sure $\epsilon_{t,i}=\bigl(\epsilon+1- \frac{1}{1+\frac{1+\epsilon}{1-\epsilon} \frac{1}{N}} \bigr) \in [0,1)$ is still well-defined, we can set $\epsilon < 1/3$.

For communication cost analysis, we bound the total number of epochs $B$ by upper bounding the total number of summations like $\sum_{s = t_{p-1}+1}^{t_{p}} \hat{\sigma}^{2}_{t_{p-1}}(\bx_{s})$, over the time horizon $NT$. Using Lemma \ref{lem:rls}, our event-trigger in Eq~\eqref{eq:sync_event} provides a lower bound $\sum_{s = t_{p-1}+1}^{t_{p}} \hat{\sigma}^{2}_{t_{p-1}}(\bx_{s}) \geq \frac{1-\epsilon}{1+\epsilon}D$. Then in order to apply the pigeonhole principle, we continue to upper bound the summation over all epochs, $\sum_{p=1}^{B} \sum_{s = t_{p-1}+1}^{t_{p}} \hat{\sigma}^{2}_{t_{p-1}}(\bx_{s}) = \sum_{p=1}^{B} \sum_{s = t_{p-1}+1}^{t_{p}} \hat{\sigma}^{2}_{s-1}(\bx_{s}) \frac{\hat{\sigma}^{2}_{t_{p-1}}(\bx_{s})}{\hat{\sigma}^{2}_{s-1}(\bx_{s})}$ by deriving a uniform bound for the ratio $\frac{\hat{\sigma}^{2}_{t_{p-1}}(\bx_{s})}{\hat{\sigma}^{2}_{s-1}(\bx_{s})} \leq \frac{\hat{\sigma}^{2}_{t_{p-1}}(\bx_{s})}{\hat{\sigma}^{2}_{t_{p}}(\bx_{s})} \leq 1 + \sum_{s=t_{p-1}+1}^{t_{p}} \hat{\sigma}^{2}_{t_{p-1}}(\bx_{s}) \leq 1 + \frac{1+\epsilon}{1-\epsilon} \sum_{s=t_{p-1}+1}^{t_{p}} \tilde{\sigma}^{2}_{t_{p-1}}(\bx_{s}) $ in terms of the communication threshold $D$ on each client. This leads to the following upper bound about the total number of epochs $B \leq \frac{1+\epsilon}{1-\epsilon}[ \frac{1}{D} + \frac{1+\epsilon}{1-\epsilon} (N + L^{2}/(\lambda D)) ] 2 \gamma_{NT}$, and with $D=1/N$, we have $C_{NT} \leq B \cdot N \gamma_{NT}^{2} = O(N^{2} \gamma_{NT}^{3}) $, which completes the proof.
\end{proof}

\begin{remark} \label{rmk:2}
Compared with \algone{}'s $O(TN^{2}d)$ communication cost, \algtwo{} removes the linear dependence on $T$, but introduces an additional $\gamma_{NT}^{3}$ dependence due to the communication of the embedded statistics. In situations where $\gamma_{NT} \ll T^{1/3} d^{1/3}$, \algone{} is preferable. 
As mentioned in Lemma \ref{lem:regret_comm_diskernelucb}, the value of $\gamma_{NT}$, which affects how much the data can be compressed,
depends on the specific arm set of the problem and the kernel function of the choice.
By Mercer's Theorem, one can represent the kernel using its eigenvalues, and $\gamma_{NT}$ characterizes how fast its eigenvalues decay. Vakili et al. \citep{vakili2021information} showed that for kernels whose eigenvalues decay exponentially, i.e., $\lambda_{m}=O(\exp(- m^{\beta_{e}}))$, for some $\beta_{e}>0$, $\gamma_{NT}=O(\log^{1+\frac{1}{\beta_{e}}}(NT))$. In this case, \algtwo{} is far more efficient than \algone{}. This includes Gaussian kernel, which is widely used for GPs and SVMs.
For kernels that have polynomially decaying eigenvalues, i.e., $\lambda_{m}=O(m^{-\beta_{p}})$, for some $\beta_{p} > 1$, $\gamma_{NT} = O(T^{\frac{1}{\beta_{p}}} \log^{1-\frac{1}{\beta_{p}}}(NT))$. Then as long as $\beta_{p} > 3$, \algtwo{} still enjoys reduced communication cost.
\end{remark}
\section{Experiments}

In order to evaluate \algtwo{}'s effectiveness in reducing communication cost, we performed extensive empirical evaluations on both synthetic and real-world datasets, and the results (averaged over 3 runs) are reported in Figure \ref{fig:synthetic_exp_results}, \ref{fig:uci_exp_results} and \ref{fig:recommendation_exp_results}, respectively.
We included \algone{}, 
DisLinUCB \citep{wang2019distributed}, OneKernelUCB, and NKernelUCB \citep{chowdhury2017kernelized} as baselines, where One-KernelUCB learns a shared bandit model across all clients' aggregated data where data aggregation happens immediately after each new data point becomes available, and N-KernelUCB learns a separated bandit model for each client with no communication. For all the kernelized algorithms, we used the Gaussian kernel $k(x, y) = \exp(-\gamma \lVert x-y \rVert^{2})$.
We did a grid search of $\gamma \in \{0.1, 1, 4\}$ for kernelized algorithms, and set $D=20$ for DisLinUCB and \algone{}, $D=5$ for \algtwo{}. For all algorithms, instead of using their theoretically derived exploration coefficient $\alpha$, we followed the convention \cite{li2010contextual,zhou2020neural} to use grid search for $\alpha$ in $\{0.1, 1, 4\}$. 
Due to space limit, here we only present the experiment results and discussions. Details about the experiment setup are presented in appendix.

\begin{figure}[t]
\centering 
\subfigure[$\cos(3 \bx^{\top} \btheta_{\star})$]{\label{fig:a}\includegraphics[width=0.35\textwidth]{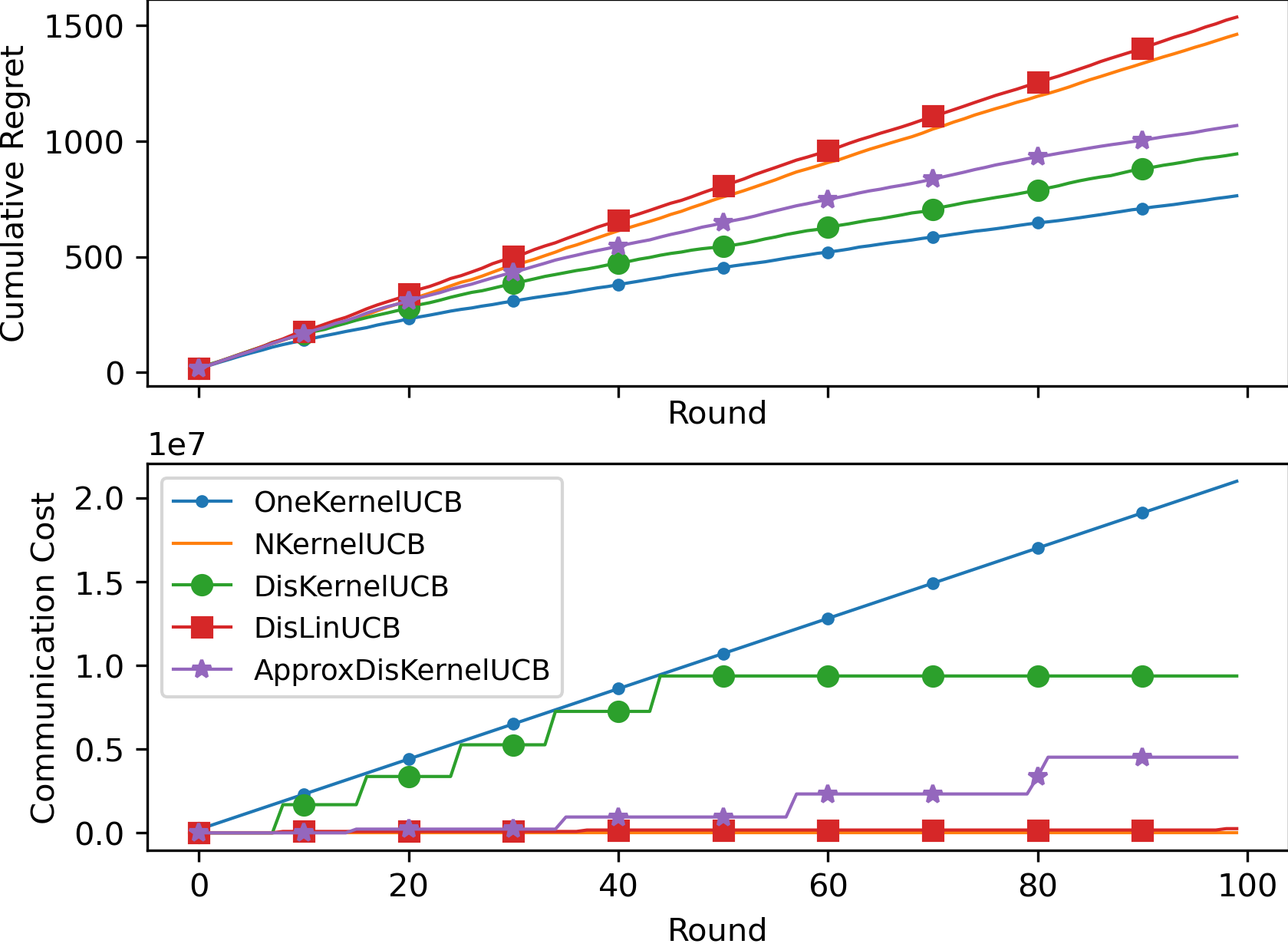}}
\subfigure[$ (\bx^{\top} \btheta_{\star})^{3} - 3(\bx^{\top} \btheta_{\star})^{2} - (\bx^{\top} \btheta_{\star}) + 3$]{\label{fig:b}\includegraphics[width=0.35\textwidth]{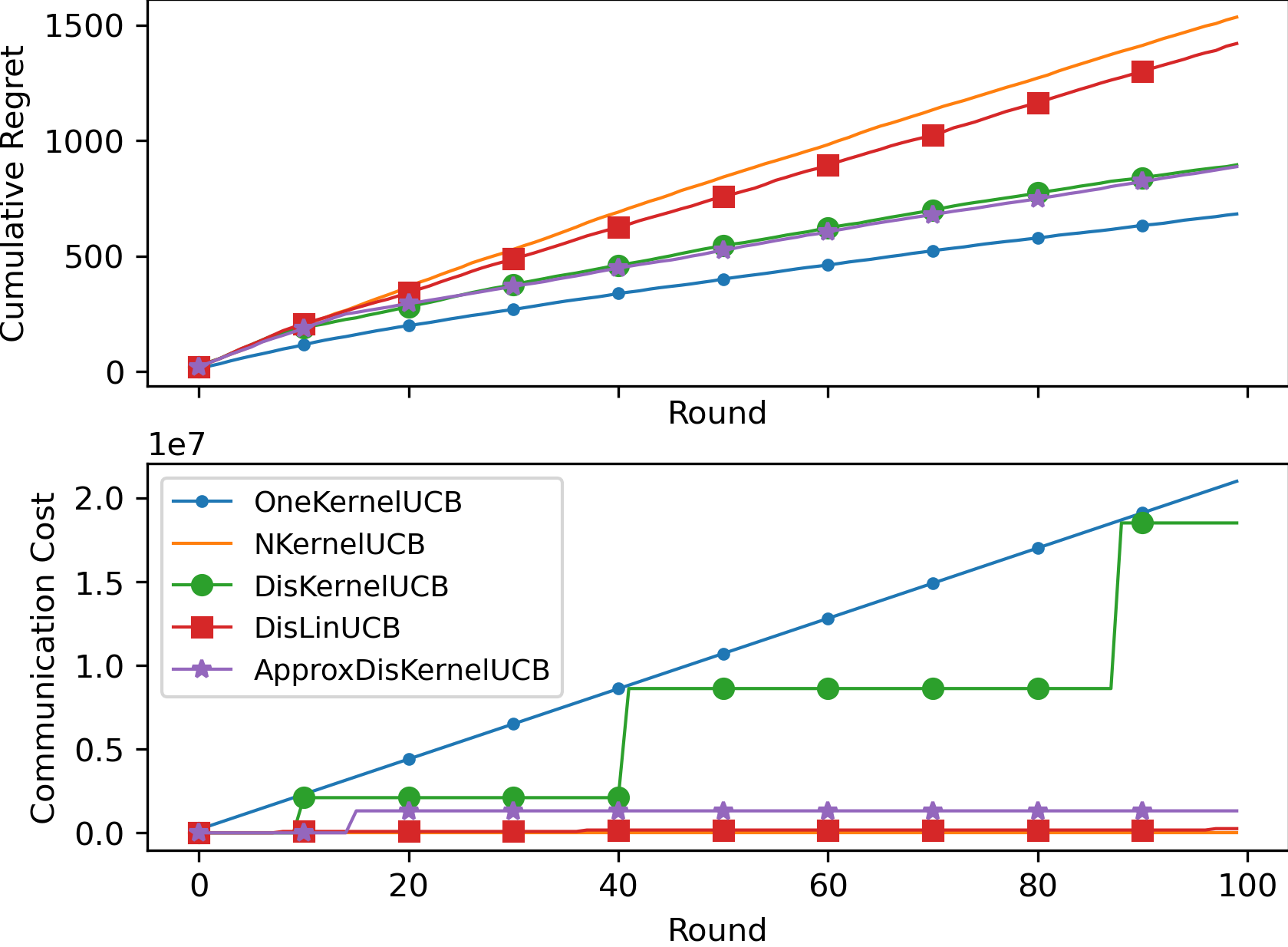}}
\caption{Experiment results on synthetic datasets with different reward function $f(\bx)$.}
\label{fig:synthetic_exp_results}
\end{figure}

\begin{figure}[t]
\centering 
\subfigure[MagicTelescope]{\label{fig:c}\includegraphics[width=0.32\textwidth]{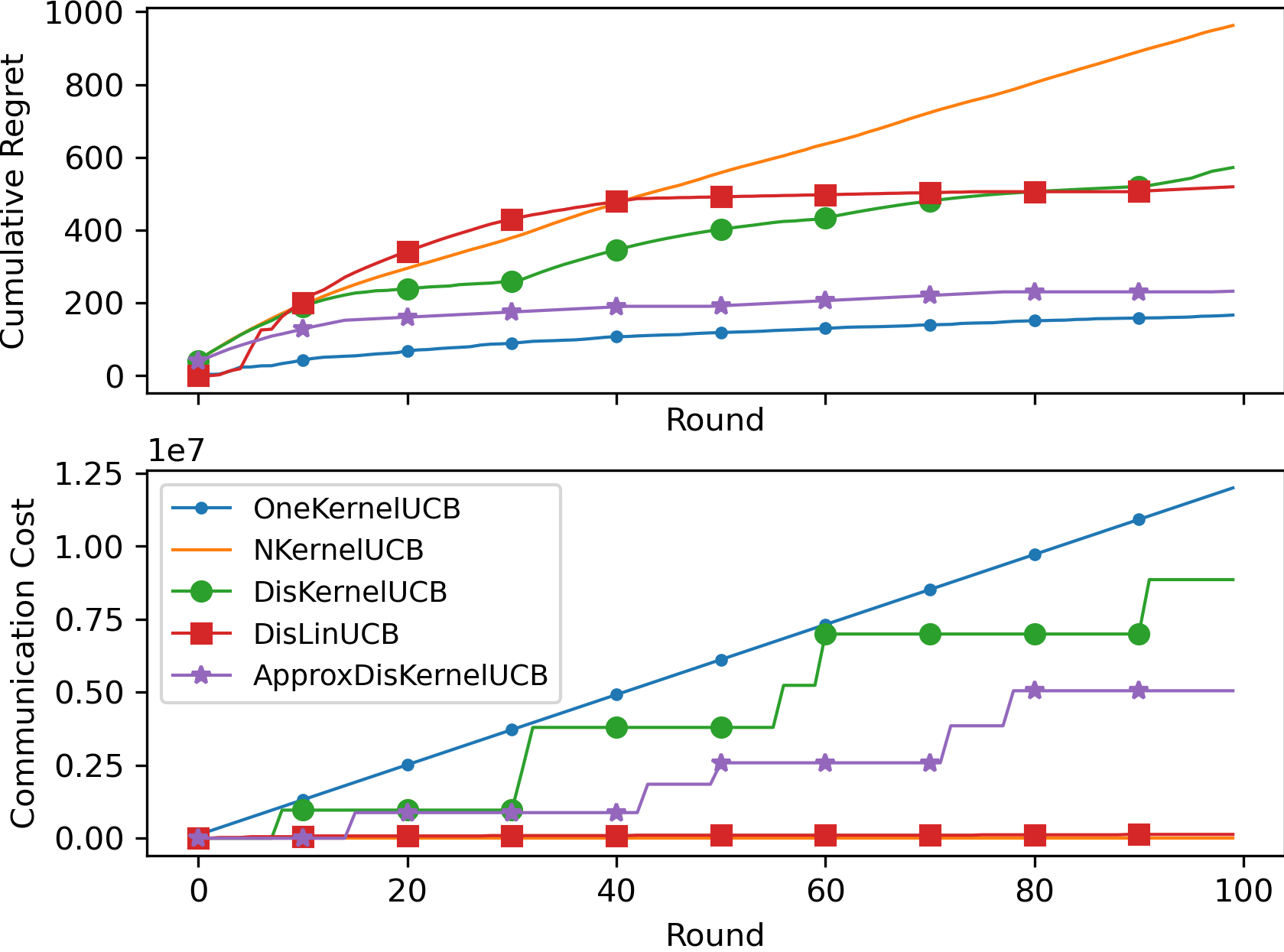}}
\subfigure[Mushroom]{\label{fig:d}\includegraphics[width=0.32\textwidth]{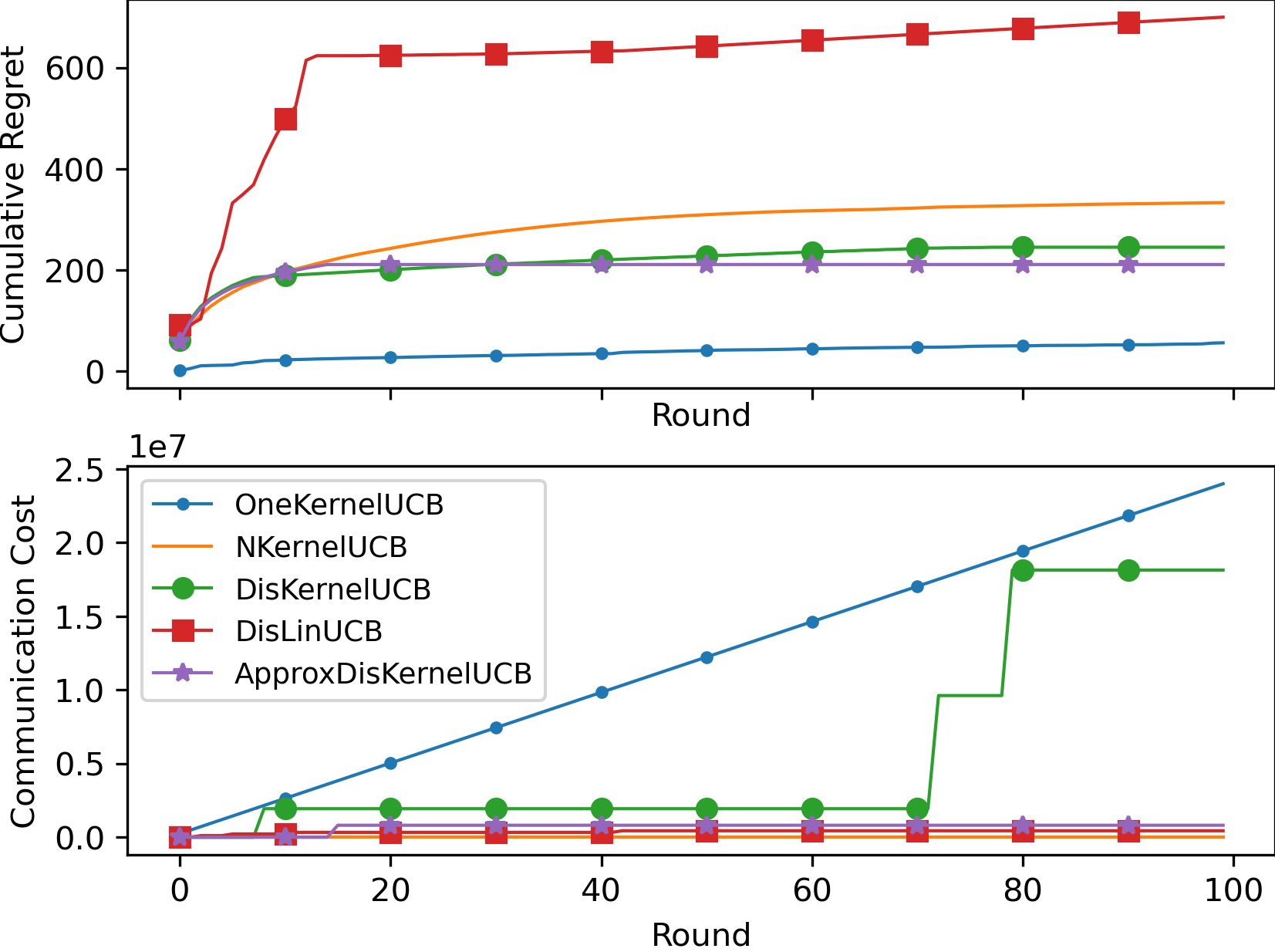}}
\subfigure[Shuttle]{\label{fig:e}\includegraphics[width=0.32\textwidth]{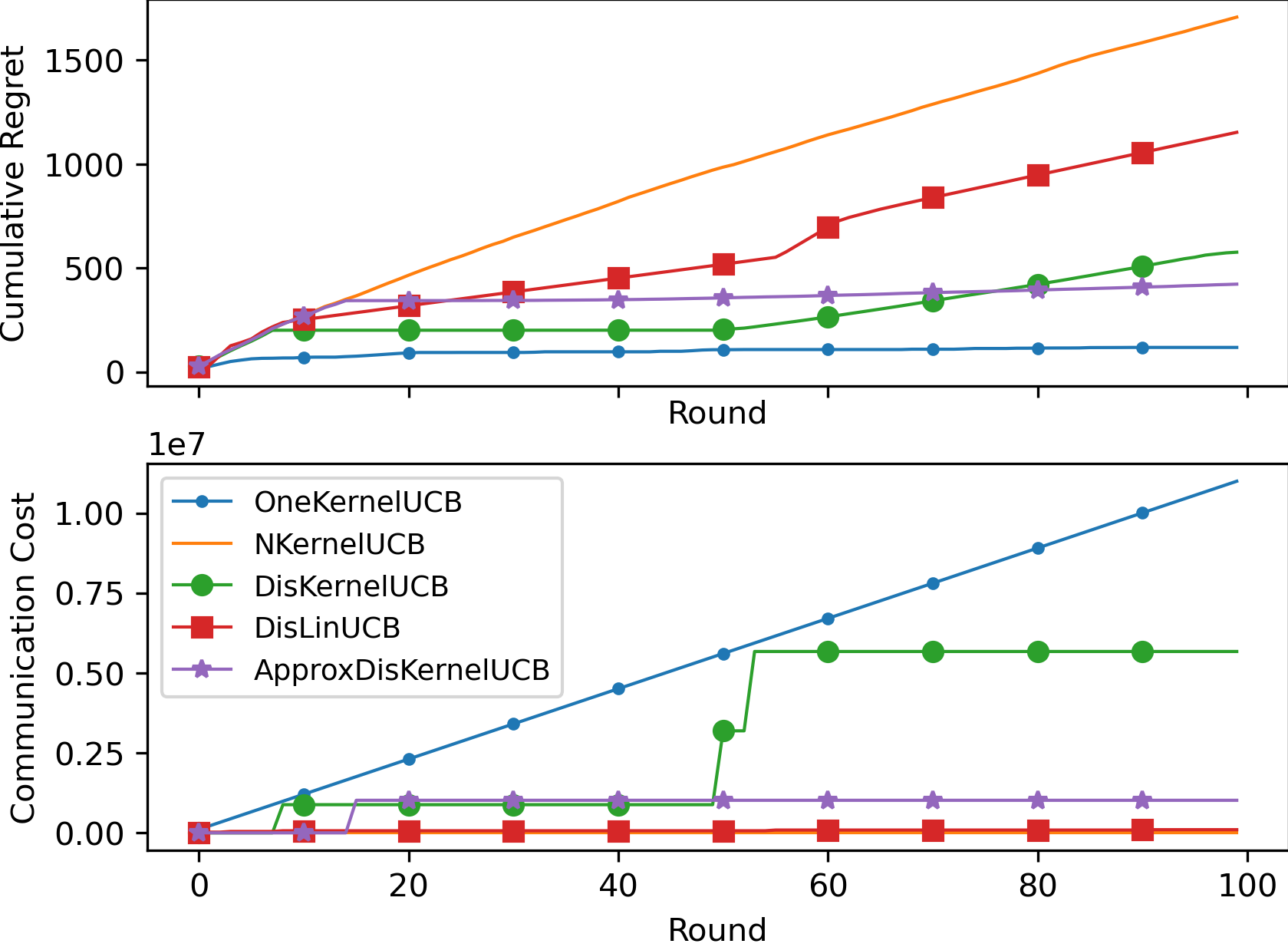}}
\caption{Experiment results on UCI datasets.}
\label{fig:uci_exp_results}
\end{figure}

\begin{figure}[h!]
\centering 
\subfigure[MovieLens]{\label{fig:f}\includegraphics[width=0.35\textwidth]{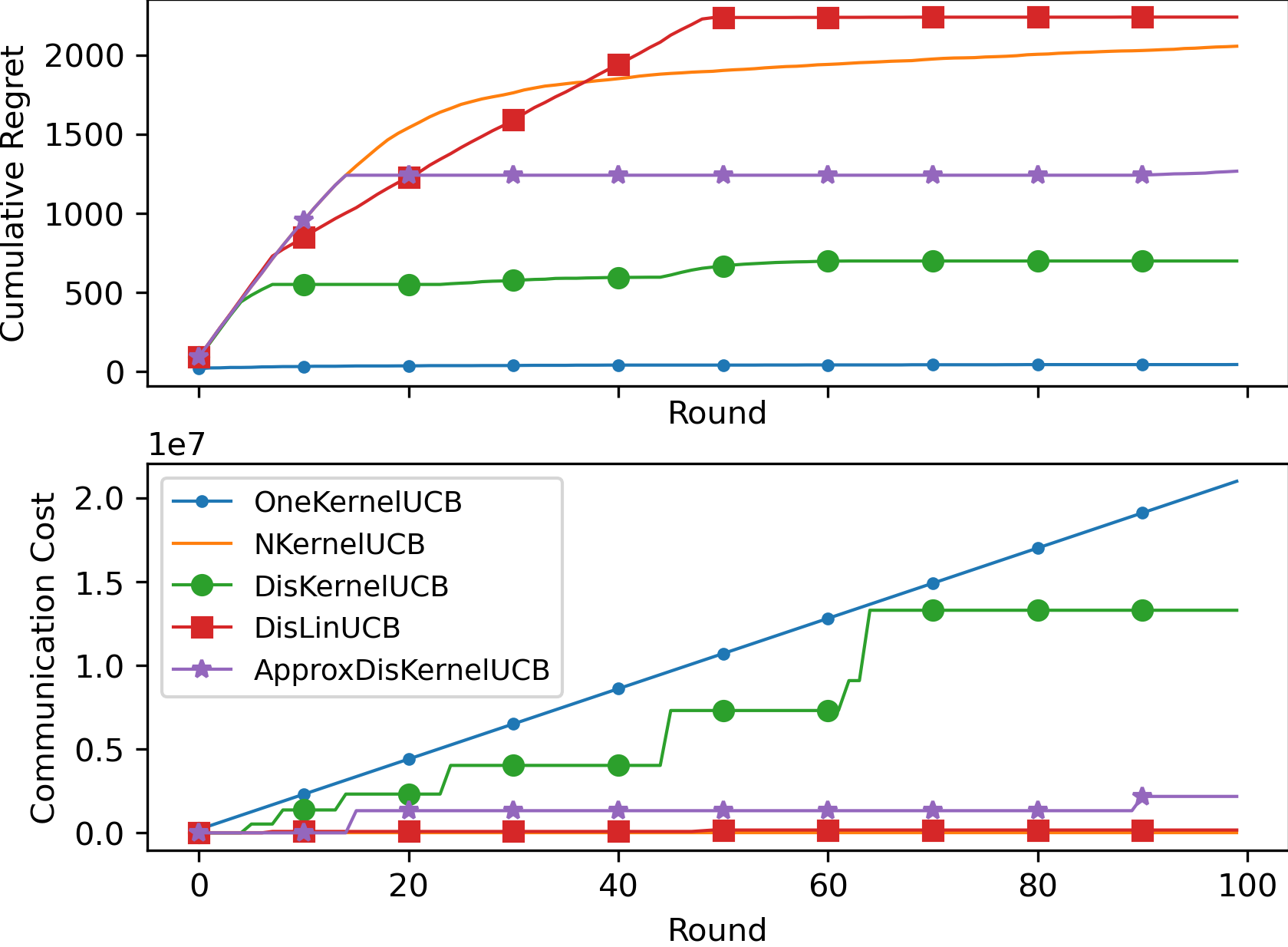}}
\subfigure[Yelp]{\label{fig:g}\includegraphics[width=0.35\textwidth]{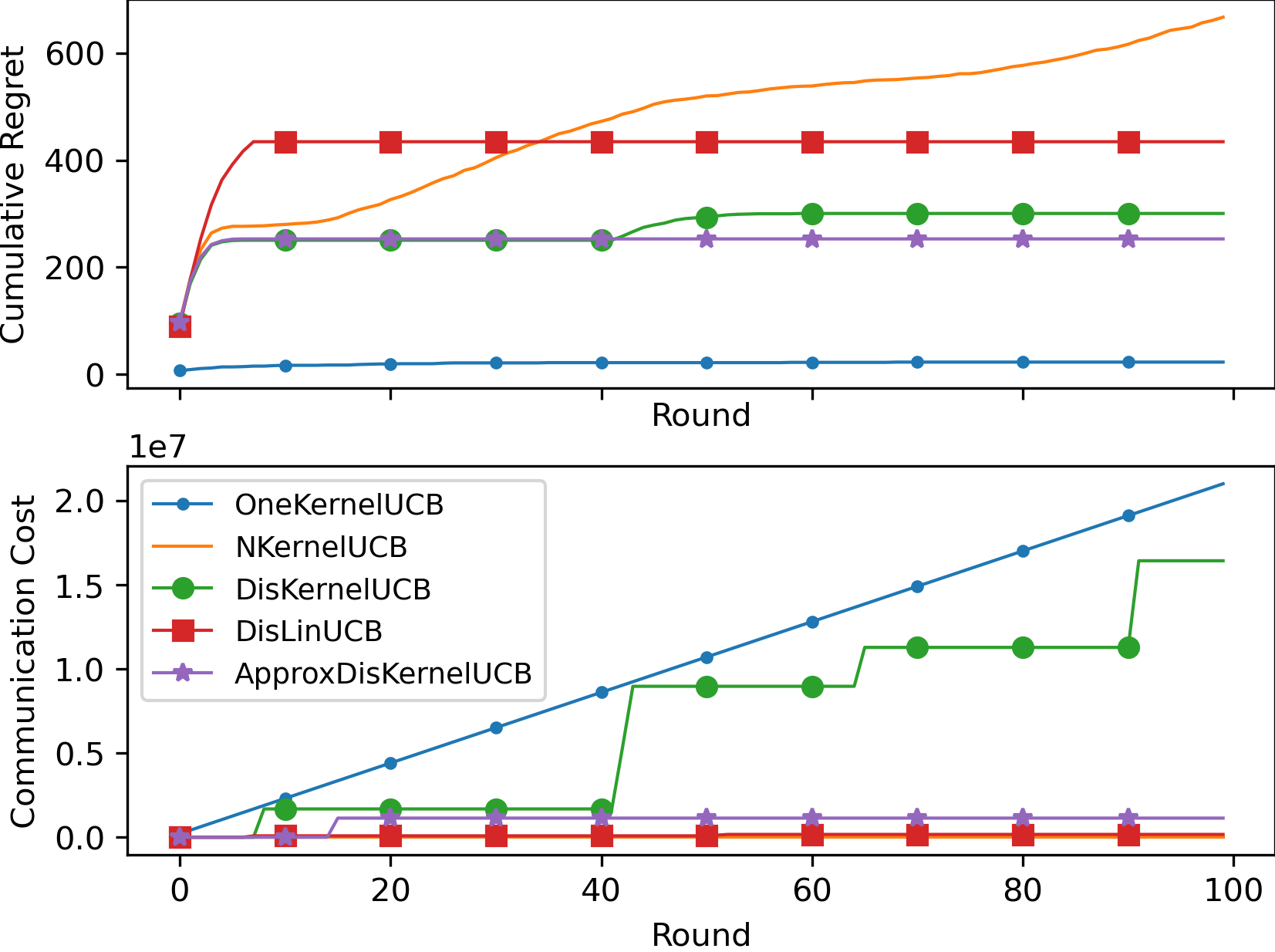}}
\caption{Experiment results on MovieLens \& Yelp datasets.}
\label{fig:recommendation_exp_results}
\end{figure}

When examining the experiment results presented in Figure \ref{fig:synthetic_exp_results}, \ref{fig:uci_exp_results} and \ref{fig:recommendation_exp_results}, we can first look at the cumulative regret and communication cost of OneKernelUCB and NKernelUCB, which correspond to the two extreme cases where the clients communicate in every time step to learn a shared model, and each client learns its own model independently with no communication, respectively.
OneKernelUCB achieves the smallest cumulative regret in all experiments, while also incurring the highest communication cost, i.e., $O(T N^{2} d)$. This demonstrates the need of efficient data aggregation across clients for reducing regret. 
Second, we can observe that \algone{} incurs the second highest communication cost in all experiments due to the transfer of raw data, as we have discussed in Remark \ref{rmk:1}, which makes it prohibitively expensive for distributed setting.
On the other extreme, we can see that DisLinUCB incurs very small communication cost thanks to its closed-form solution, but fails to capture the complicated reward mappings in most of these datasets, e.g. in Figure \ref{fig:a}, \ref{fig:d} and \ref{fig:f}, it leads to even worse regret than NKernelUCB that learns a kernelized bandit model independently for each client.
In comparison, the proposed \algtwo{} algorithm enjoys the best of both worlds in most cases, i.e., it can take advantage of the superior modeling power of kernels to reduce regret, while only requiring a relatively low communication cost for clients to collaborate. On all the datasets, \algtwo{} achieved comparable regret with \algone{} that maintains exact kernelized estimators, and sometimes even getting very close to OneKernelUCB, e.g., in Figure \ref{fig:b} and \ref{fig:c}, but its communication cost is only slightly higher than that of DisLinUCB.

\section{Conclusion}
In this paper, we proposed the first communication efficient algorithm for distributed kernel bandits using Nystr\"{o}m approximation. Clients in the learning system project their local data to a finite RKHS spanned by a shared dictionary, and then communicate the embedded statistics for collaborative exploration. To ensure communication efficiency, the frequency of dictionary update and synchronization of embedded statistics are controlled by an event-trigger.
The algorithm is proved to incur $O(N^{2} \gamma_{NT}^{3})$ communication cost, while attaining the optimal $O(\sqrt{NT}\gamma_{NT})$ cumulative regret.

We should note that the total number of synchronizations required by \algtwo{} is $N \gamma_{NT}$, which is $\sqrt{N}$ worse than \algone{}. 
An important future direction of this work is to investigate whether this part can be further improved.
It is also interesting to extend the proposed algorithm to
P2P setting, at the absence of a central server to coordinate the update of the shared dictionary and the exchange of embedded statistics. 
Due to the delay in propagating messages, it may be beneficial to utilize possible local structures in the network of clients, and approximate each block of the kernel matrix separately \citep{si2014memory}, i.e., each block corresponds to a group of clients, instead of directly approximating the complete matrix.

\section{Acknowledgement}
This work is supported by NSF grants IIS-2128019, IIS-1838615, IIS-1553568, IIS-2107304, CMMI-1653435, AFOSR grant and ONR grant 1006977.

\bibliography{bibfile}

\begin{thebibliography}{32}
\providecommand{\natexlab}[1]{#1}
\providecommand{\url}[1]{\texttt{#1}}
\expandafter\ifx\csname urlstyle\endcsname\relax
  \providecommand{\doi}[1]{doi: #1}\else
  \providecommand{\doi}{doi: \begingroup \urlstyle{rm}\Url}\fi

\bibitem[Abbasi-Yadkori et~al.(2011)Abbasi-Yadkori, P{\'a}l, and
  Szepesv{\'a}ri]{abbasi2011improved}
Yasin Abbasi-Yadkori, D{\'a}vid P{\'a}l, and Csaba Szepesv{\'a}ri.
\newblock Improved algorithms for linear stochastic bandits.
\newblock \emph{Advances in neural information processing systems},
  24:\penalty0 2312--2320, 2011.

\bibitem[Ban et~al.(2021)Ban, Yan, Banerjee, and He]{ban2021ee}
Yikun Ban, Yuchen Yan, Arindam Banerjee, and Jingrui He.
\newblock Ee-net: Exploitation-exploration neural networks in contextual
  bandits.
\newblock \emph{arXiv preprint arXiv:2110.03177}, 2021.

\bibitem[Calandriello et~al.(2017)Calandriello, Lazaric, and
  Valko]{calandriello2017second}
Daniele Calandriello, Alessandro Lazaric, and Michal Valko.
\newblock Second-order kernel online convex optimization with adaptive
  sketching.
\newblock In \emph{International Conference on Machine Learning}, pages
  645--653. PMLR, 2017.

\bibitem[Calandriello et~al.(2019)Calandriello, Carratino, Lazaric, Valko, and
  Rosasco]{calandriello2019gaussian}
Daniele Calandriello, Luigi Carratino, Alessandro Lazaric, Michal Valko, and
  Lorenzo Rosasco.
\newblock Gaussian process optimization with adaptive sketching: Scalable and
  no regret.
\newblock In \emph{Conference on Learning Theory}, pages 533--557. PMLR, 2019.

\bibitem[Calandriello et~al.(2020)Calandriello, Carratino, Lazaric, Valko, and
  Rosasco]{calandriello2020near}
Daniele Calandriello, Luigi Carratino, Alessandro Lazaric, Michal Valko, and
  Lorenzo Rosasco.
\newblock Near-linear time gaussian process optimization with adaptive batching
  and resparsification.
\newblock In \emph{International Conference on Machine Learning}, pages
  1295--1305. PMLR, 2020.

\bibitem[Chowdhury and Gopalan(2017)]{chowdhury2017kernelized}
Sayak~Ray Chowdhury and Aditya Gopalan.
\newblock On kernelized multi-armed bandits.
\newblock In \emph{International Conference on Machine Learning}, pages
  844--853. PMLR, 2017.

\bibitem[Du et~al.(2021)Du, Chen, Yuroki, and Huang]{du2021collaborative}
Yihan Du, Wei Chen, Yuko Yuroki, and Longbo Huang.
\newblock Collaborative pure exploration in kernel bandit.
\newblock \emph{arXiv preprint arXiv:2110.15771}, 2021.

\bibitem[Dua and Graff(2017)]{Dua:2019}
Dheeru Dua and Casey Graff.
\newblock {UCI} machine learning repository, 2017.
\newblock URL \url{http://archive.ics.uci.edu/ml}.

\bibitem[Dubey and Pentland(2020)]{dubey2020differentially}
Abhimanyu Dubey and AlexSandy' Pentland.
\newblock Differentially-private federated linear bandits.
\newblock \emph{Advances in Neural Information Processing Systems}, 33, 2020.

\bibitem[Dubey et~al.(2020)]{dubey2020kernel}
Abhimanyu Dubey et~al.
\newblock Kernel methods for cooperative multi-agent contextual bandits.
\newblock In \emph{International Conference on Machine Learning}, pages
  2740--2750. PMLR, 2020.

\bibitem[Durand et~al.(2018)Durand, Achilleos, Iacovides, Strati, Mitsis, and
  Pineau]{durand2018contextual}
Audrey Durand, Charis Achilleos, Demetris Iacovides, Katerina Strati,
  Georgios~D Mitsis, and Joelle Pineau.
\newblock Contextual bandits for adapting treatment in a mouse model of de novo
  carcinogenesis.
\newblock In \emph{Machine learning for healthcare conference}, pages 67--82.
  PMLR, 2018.

\bibitem[Filippi et~al.(2010)Filippi, Cappe, Garivier, and
  Szepesv{\'a}ri]{filippi2010parametric}
Sarah Filippi, Olivier Cappe, Aur{\'e}lien Garivier, and Csaba Szepesv{\'a}ri.
\newblock Parametric bandits: The generalized linear case.
\newblock In \emph{NIPS}, volume~23, pages 586--594, 2010.

\bibitem[Harper and Konstan(2015)]{harper2015movielens}
F~Maxwell Harper and Joseph~A Konstan.
\newblock The movielens datasets: History and context.
\newblock \emph{Acm transactions on interactive intelligent systems (tiis)},
  5\penalty0 (4):\penalty0 1--19, 2015.

\bibitem[He et~al.(2022)He, Wang, Min, and Gu]{he2022simple}
Jiafan He, Tianhao Wang, Yifei Min, and Quanquan Gu.
\newblock A simple and provably efficient algorithm for asynchronous federated
  contextual linear bandits.
\newblock \emph{arXiv preprint arXiv:2207.03106}, 2022.

\bibitem[Hillel et~al.(2013)Hillel, Karnin, Koren, Lempel, and
  Somekh]{hillel2013distributed}
Eshcar Hillel, Zohar~S Karnin, Tomer Koren, Ronny Lempel, and Oren Somekh.
\newblock Distributed exploration in multi-armed bandits.
\newblock \emph{Advances in Neural Information Processing Systems}, 26, 2013.

\bibitem[Huang et~al.(2013)Huang, Qian, Guo, Zhou, Xu, Mao, Sen, and
  Spatscheck]{huang2013depth}
Junxian Huang, Feng Qian, Yihua Guo, Yuanyuan Zhou, Qiang Xu, Z~Morley Mao,
  Subhabrata Sen, and Oliver Spatscheck.
\newblock An in-depth study of lte: Effect of network protocol and application
  behavior on performance.
\newblock \emph{ACM SIGCOMM Computer Communication Review}, 43\penalty0
  (4):\penalty0 363--374, 2013.

\bibitem[Huang et~al.(2021)Huang, Wu, Yang, and Shen]{huang2021federated}
Ruiquan Huang, Weiqiang Wu, Jing Yang, and Cong Shen.
\newblock Federated linear contextual bandits.
\newblock \emph{Advances in Neural Information Processing Systems}, 34, 2021.

\bibitem[Korda et~al.(2016)Korda, Szorenyi, and Li]{korda2016distributed}
Nathan Korda, Balazs Szorenyi, and Shuai Li.
\newblock Distributed clustering of linear bandits in peer to peer networks.
\newblock In \emph{International conference on machine learning}, pages
  1301--1309. PMLR, 2016.

\bibitem[Li and Wang(2022)]{li2022asynchronous}
Chuanhao Li and Hongning Wang.
\newblock Asynchronous upper confidence bound algorithms for federated linear
  bandits.
\newblock In \emph{International Conference on Artificial Intelligence and
  Statistics}, pages 6529--6553. PMLR, 2022.

\bibitem[Li et~al.(2010{\natexlab{a}})Li, Chu, Langford, and
  Schapire]{li2010contextual}
Lihong Li, Wei Chu, John Langford, and Robert~E Schapire.
\newblock A contextual-bandit approach to personalized news article
  recommendation.
\newblock In \emph{Proceedings of the 19th international conference on World
  wide web}, pages 661--670, 2010{\natexlab{a}}.

\bibitem[Li et~al.(2010{\natexlab{b}})Li, Wang, Zhang, Cui, Mao, and
  Jin]{li2010exploitation}
Wei Li, Xuerui Wang, Ruofei Zhang, Ying Cui, Jianchang Mao, and Rong Jin.
\newblock Exploitation and exploration in a performance based contextual
  advertising system.
\newblock In \emph{Proceedings of the 16th ACM SIGKDD international conference
  on Knowledge discovery and data mining}, pages 27--36, 2010{\natexlab{b}}.

\bibitem[Nystr{\"o}m(1930)]{nystrom1930praktische}
Evert~J Nystr{\"o}m.
\newblock {\"U}ber die praktische aufl{\"o}sung von integralgleichungen mit
  anwendungen auf randwertaufgaben.
\newblock \emph{Acta Mathematica}, 54:\penalty0 185--204, 1930.

\bibitem[Scarlett et~al.(2017)Scarlett, Bogunovic, and
  Cevher]{scarlett2017lower}
Jonathan Scarlett, Ilija Bogunovic, and Volkan Cevher.
\newblock Lower bounds on regret for noisy gaussian process bandit
  optimization.
\newblock In \emph{Conference on Learning Theory}, pages 1723--1742. PMLR,
  2017.

\bibitem[Shi and Shen(2021)]{shi2021federated}
Chengshuai Shi and Cong Shen.
\newblock Federated multi-armed bandits.
\newblock In \emph{Proceedings of the 35th AAAI Conference on Artificial
  Intelligence (AAAI)}, 2021.

\bibitem[Si et~al.(2014)Si, Hsieh, and Dhillon]{si2014memory}
Si~Si, Cho-Jui Hsieh, and Inderjit Dhillon.
\newblock Memory efficient kernel approximation.
\newblock In \emph{International Conference on Machine Learning}, pages
  701--709. PMLR, 2014.

\bibitem[Srinivas et~al.(2009)Srinivas, Krause, Kakade, and
  Seeger]{srinivas2009gaussian}
Niranjan Srinivas, Andreas Krause, Sham~M Kakade, and Matthias Seeger.
\newblock Gaussian process optimization in the bandit setting: No regret and
  experimental design.
\newblock \emph{arXiv preprint arXiv:0912.3995}, 2009.

\bibitem[Tao et~al.(2019)Tao, Zhang, and Zhou]{tao2019collaborative}
Chao Tao, Qin Zhang, and Yuan Zhou.
\newblock Collaborative learning with limited interaction: Tight bounds for
  distributed exploration in multi-armed bandits.
\newblock In \emph{2019 IEEE 60th Annual Symposium on Foundations of Computer
  Science (FOCS)}, pages 126--146. IEEE, 2019.

\bibitem[Vakili et~al.(2021)Vakili, Khezeli, and
  Picheny]{vakili2021information}
Sattar Vakili, Kia Khezeli, and Victor Picheny.
\newblock On information gain and regret bounds in gaussian process bandits.
\newblock In \emph{International Conference on Artificial Intelligence and
  Statistics}, pages 82--90. PMLR, 2021.

\bibitem[Valko et~al.(2013)Valko, Korda, Munos, Flaounas, and
  Cristianini]{valko2013finite}
Michal Valko, Nathaniel Korda, R{\'e}mi Munos, Ilias Flaounas, and Nelo
  Cristianini.
\newblock Finite-time analysis of kernelised contextual bandits.
\newblock \emph{arXiv preprint arXiv:1309.6869}, 2013.

\bibitem[Wang et~al.(2019)Wang, Hu, Chen, and Wang]{wang2019distributed}
Yuanhao Wang, Jiachen Hu, Xiaoyu Chen, and Liwei Wang.
\newblock Distributed bandit learning: Near-optimal regret with efficient
  communication.
\newblock In \emph{International Conference on Learning Representations}, 2019.

\bibitem[Zenati et~al.(2022)Zenati, Bietti, Diemert, Mairal, Martin, and
  Gaillard]{zenati2022efficient}
Houssam Zenati, Alberto Bietti, Eustache Diemert, Julien Mairal, Matthieu
  Martin, and Pierre Gaillard.
\newblock Efficient kernel ucb for contextual bandits.
\newblock \emph{arXiv preprint arXiv:2202.05638}, 2022.

\bibitem[Zhou et~al.(2020)Zhou, Li, and Gu]{zhou2020neural}
Dongruo Zhou, Lihong Li, and Quanquan Gu.
\newblock Neural contextual bandits with ucb-based exploration.
\newblock In \emph{International Conference on Machine Learning}, pages
  11492--11502. PMLR, 2020.

\end{thebibliography}

\clearpage
\appendix
\section{Technical Lemmas}


\begin{lemma}[Lemma 12 of \citep{abbasi2011improved}] \label{lem:quadratic_det_inequality}
Let $A$, $B$ and $C$ be positive semi-definite matrices with finite dimension, such that $A=B+C$. Then, we have that:
\begin{align*}
    \sup_{\bx \neq \textbf{0}} \frac{\bx^{\top} A \bx}{\bx^{\top} B \bx} \leq \frac{\det(A)}{\det(B)}
\end{align*}
\end{lemma}
\begin{lemma}[Extension of Lemma \ref{lem:quadratic_det_inequality} to kernel matrix] \label{lem:quadratic_det_inequality_infinite}
Define positive definite matrices $A=\lambda \bI + \bPhi_{1}^{\top}\bPhi_{1} + \bPhi_{2}^{\top}\bPhi_{2}$ and $B=\lambda \bI + \bPhi_{1}^{\top}\bPhi_{1}$, where $\bPhi_{1}^{\top}\bPhi_{1},\bPhi_{2}^{\top}\bPhi_{2} \in \bR^{p \times p}$ and $p$ is possibly infinite. Then, we have that:
\begin{align*}
    \sup_{\phi \neq \textbf{0}} \frac{\phi^{\top} A \phi}{\phi^{\top} B \phi} \leq \frac{\det(\bI+ \lambda^{-1} \bK_{A})}{\det(\bI + \lambda^{-1} \bK_{B})}
\end{align*}
where $\bK_{A}=\begin{bmatrix} \bPhi_{1}\\\bPhi_{2} \end{bmatrix} \begin{bmatrix} \bPhi_{1}^{\top},\bPhi_{2}^{\top} \end{bmatrix}$ and $\bK_{B}=\bPhi_{1} \bPhi_{1}^{\top}$.
\end{lemma}
\begin{proof}[Proof of Lemma \ref{lem:quadratic_det_inequality_infinite}]
Similar to the proof of Lemma 12 of \citep{abbasi2011improved}, we start from the simple case when $ \bPhi_{2}^{\top}\bPhi_{2}=m m^{\top}$, where $m \in \bR^{p}$. Using Cauchy-Schwartz inequality, we have
\begin{align*}
    (\phi^{\top} m)^{2} = (\phi^{\top}B^{1/2}B^{-1/2}m)^{2} \leq \lVert B^{1/2} \phi \rVert^{2} \lVert B^{-1/2} m \rVert^{2} = \lVert \phi \rVert^{2}_{B} \lVert m \rVert^{2}_{B^{-1}},
\end{align*}
and thus,
\begin{align*}
    \phi^{\top} (B + m m^{\top}) \phi \leq \phi^{\top} B \phi + \lVert \phi \rVert^{2}_{B} \lVert m \rVert^{2}_{B^{-1}} = (1+ \lVert m \rVert^{2}_{B^{-1}}) \lVert \phi \rVert^{2}_{B},
\end{align*}
so we have
\begin{align*}
    \frac{\phi^{\top} A \phi}{\phi^{\top} B \phi} \leq 1 + \lVert m \rVert_{B^{-1}}^{2}
\end{align*}
for any $\phi$. Then using the kernel trick, e.g., see the derivation of Eq (27) in \cite{zenati2022efficient}, we have
\begin{align*}
    1 + \lVert m \rVert_{B^{-1}}^{2} = \frac{\det(\bI+ \lambda^{-1} \bK_{A})}{\det(\bI + \lambda^{-1} \bK_{B})},
\end{align*}
which finishes the proof of this simple case.
Now consider the general case where $\bPhi_{2}^{\top} \bPhi_{2} = m_{1} m_{1}^{\top} + m_{2} m_{2}^{\top} + \dots + m_{t-1} m_{t-1}^{\top}$. Let's define $V_{s}=B+m_{1} m_{1}^{\top} + m_{2} m_{2}^{\top} + \dots + m_{s-1} m_{s-1}^{\top}$ and the corresponding kernel matrix $\bK_{V_{s}}=\begin{bmatrix} \bPhi_{1}\\ m_{1}^{\top}  \\ \dots \\ m_{s-1}^{\top} \end{bmatrix} \begin{bmatrix} \bPhi_{1}^{\top}, m_{1}, \dots, m_{s-1} \end{bmatrix}$, and note that $\frac{\phi^{\top} A \phi}{\phi^{\top} B \phi} = \frac{\phi^{\top} V_{t} \phi}{\phi^{\top} V_{t-1} \phi} \frac{\phi^{\top} V_{t-1} \phi}{\phi^{\top} V_{t-2} \phi} \dots \frac{\phi^{\top} V_{2} \phi}{\phi^{\top} B \phi}$.
Then we can apply the result for the simple case on each term in the product above, which gives us
\begin{align*}
    \frac{\phi^{\top} A \phi}{\phi^{\top} B \phi} & \leq \frac{\det(\bI+ \lambda^{-1} \bK_{V_{t}})}{\det(\bI + \lambda^{-1} \bK_{V_{t-1}})} \frac{\det(\bI+ \lambda^{-1} \bK_{V_{t-1}})}{\det(\bI + \lambda^{-1} \bK_{V_{t-2}})} \dots \frac{\det(\bI+ \lambda^{-1} \bK_{V_{2}})}{\det(\bI + \lambda^{-1} \bK_{B})} \\
    & = \frac{\det(\bI+ \lambda^{-1} \bK_{V_{t}})}{\det(\bI + \lambda^{-1} \bK_{B})} = \frac{\det(\bI+ \lambda^{-1} \bK_{A})}{\det(\bI + \lambda^{-1} \bK_{B})},
\end{align*}
which finishes the proof.

\end{proof}

\begin{lemma}[Eq (26) and Eq (27) of \cite{zenati2022efficient}] \label{lem:sum_sqr_log_det}
Let $\{\phi_{t}\}_{t=1}^{\infty}$ be a sequence in $\bR^{p}$, $V \in \bR^{p \times p}$ a positive definite matrix, where $p$ is possibly infinite, and define $V_{t}=V + \sum_{s=1}^{t} \phi_{s} \phi_{s}^{\top}$. Then we have that
\begin{align*}
    \sum_{t=1}^{n} \min\bigl(\lVert \phi_{t} \rVert^{2}_{V_{t-1}^{-1}}, 1\bigr) \leq 2 \ln\bigl( \det(\bI + \lambda^{-1} \bK_{V_{t}})\bigr),
\end{align*}
where $\bK_{V_{t}}$ is the kernel matrix corresponding to $V_{t}$ as defined in Lemma \ref{lem:quadratic_det_inequality_infinite}.
\end{lemma}

\begin{lemma}[Lemma 4 of \citep{calandriello2020near}] \label{lem:bound_variance_ratio}
For $t > t_\text{last}$, we have for any $\bx \in \bR^{d}$
\begin{align*}
    \hat{\sigma}_{t}^{2}(\bx) \leq \hat{\sigma}_{t_\text{last}}^{2}(\bx) \leq \bigl( 1+\sum_{s=t_\text{last}+1}^{t} \hat{\sigma}_{t_\text{last}}^{2}(\bx_{s}) \bigr) \hat{\sigma}_{t}^{2}(\bx)
\end{align*}
\end{lemma}

\section{Confidence Ellipsoid for DisKernelUCB}
In this section, we construct the confidence ellipsoid for DisKernelUCB as shown in Lemma \ref{lem:confidence_ellipsoid_diskernelucb}.
\begin{lemma}[Confidence Ellipsoid for DisKernelUCB] \label{lem:confidence_ellipsoid_diskernelucb}
Let $\delta \in (0,1)$. With probability at least $1-\delta$, for all $t \in [NT], i \in [N]$, we have
\begin{align*}
    \lVert \hat{\theta}_{t,i} - \theta_{\star} \rVert_{\bA_{t,i}} \leq \sqrt{\lambda} \lVert \theta_{\star} \rVert + R\sqrt{ 2 \ln(N/\delta) + \ln( \det( \bK_{ \cD_{t}(i), \cD_{t}(i) }/\lambda + \bI) ) }.
\end{align*}
\end{lemma}

The analysis is rooted in \citep{zenati2022efficient} for kernelized contextual bandit, but with non-trivial extensions: we adopted the stopping time argument from \citep{abbasi2011improved} to remove a logarithmic factor in $T$ (this improvement is hinted in Section 3.3 of \citep{zenati2022efficient} as well); and this stopping time argument is based on a special `batched filtration' that is different for each client, which is required to address the dependencies due to the event-triggered distributed communication. This problem also exists in prior works of distributed linear bandit, but was not addressed rigorously (see Lemma H.1. of \citep{wang2019distributed}).

Recall that the Ridge regression estimator 
\begin{align*}
    \hat{\theta}_{t,i} & = \bA_{t,i}^{-1} \sum_{s \in \cD_{t}(i)} \phi_{s} y_{s} =  \bA_{t,i}^{-1} \sum_{s \in \cD_{t}(i)} \phi_{s} ( \phi_{s} ^{\top} \theta_{\star} + \eta_{s} )  \\
    & = \theta_{\star} - \lambda \bA_{t,i}^{-1} \theta_{\star} + \bA_{t,i}^{-1} \sum_{s \in \cD_{t}(i)} \phi_{s} \eta_{s},
\end{align*}
and thus, we have
\begin{equation} \label{eq:ellipsoid_intermediate}
\begin{split}
    \lVert \bA_{t,i}^{1/2} (\hat{\theta}_{t,i} - \theta_{\star}) \rVert & = \lVert -\lambda \bA_{t,i}^{-1/2} \theta_{\star} + \bA_{t,i}^{-1/2} \sum_{s \in \cD_{t}(i)} \phi_{s} \eta_{s} \rVert  \\
    & \leq \lVert \lambda \bA_{t,i}^{-1/2} \theta_{\star} \rVert + \lVert \bA_{t,i}^{-1/2} \sum_{s \in \cD_{t}(i)} \phi_{s} \eta_{s} \rVert   \\
    & \leq \sqrt{\lambda} \lVert \theta_{\star} \rVert + \lVert \bA_{t,i}^{-1/2} \sum_{s \in \cD_{t}(i)} \phi_{s} \eta_{s} \rVert
\end{split}
\end{equation}
where the first inequality is due to the triangle inequality, and the second is due to the property of Rayleigh quotient, i.e., $\lVert \bA_{t,i}^{-1/2} \theta_{\star} \rVert \leq \lVert \theta_{\star} \rVert \sqrt{\lambda_{max}(\bA_{t,i}^{-1}) } \leq \lVert \theta_{\star} \rVert \frac{1}{\sqrt{\lambda}} $.

\paragraph{Difference from standard argument}
Note that the second term may seem similar to the ones appear in the self-normalized bound in previous works of linear and kernelized bandits \citep{abbasi2011improved,chowdhury2017kernelized,zenati2022efficient}. However, a main difference is that $\cD_{t}(i)$, i.e., the sequence of indices for the data points used to update client $i$, is constructed using the event-trigger as defined in Eq~\eqref{eq:sync_event_exact}
. The event-trigger is data-dependent, and thus it is a delayed and permuted version of the original sequence $[t]$. It is delayed in the sense that the length $|\cD_{t}(i)| < t$ unless $t$ is the global synchronization step. 
It is permuted in the sense that every client receives the data in a different order, i.e., before the synchronization, each client first updates using its local new data, and then receives data from other clients at the synchronization. This prevents us from directly applying Lemma 3.1 of \citep{zenati2022efficient}, and requires a careful construction of the filtration, as shown in the following paragraph.

\paragraph{Construction of filtration}
For some client $i$ at time step $t$, the sequence of time indices in $\cD_{t}(i)$ is arranged in the order that client $i$ receives the corresponding data points, which includes both data added during local update in each client, and data received from the server during global synchronization. The former adds one data point at a time, while the latter adds a batch of data points, which, as we will see, break the assumption commonly made in standard self-normalized bound \citep{abbasi2011improved,chowdhury2017kernelized,zenati2022efficient}.
Specifically, we denote $\cD_{t}(i)[k]$, for $k \leq |\cD_{t}(i)|$, as the $k$-th element in this sequence, i.e., $(\bx_{\cD_{t}(i)[k]}, y_{\cD_{t}(i)[k]})$ is the $k$-th data point received by client $i$. 
Then we denote $\cB_{t}(i)=\{k_{1}, k_{2},\dots\}$ as the sequence of $k$'s that marks the end of each batch (a singel data point added by local update is also considered a batch).
We can see that if the $k$-th element is in the middle of a batch, i.e., $k \in [k_{q-1}, k_{q}]$, it has dependency all the way to the $k_{q}$'s element, since this index can only be determined until some client triggers a global synchronization at time step $\cD_{t}(i)[k_{q}]$.


Denote by $\cF_{k,i}=\sigma \bigl( (\bx_{s}, \eta_{s})_{s \in \cD_{\infty}(i)[1:k-1]}, \bx_{\cD_{\infty}(i)[k]} \bigr)$ the $\sigma$-algebra generated by the sequence of data points up to the $k$-th element in $\cD_{\infty}(i)$. 
As we mentioned, because of the dependency of the index on the future data points, for some $k$-th element that is inside a batch, i.e., $k \in [k_{q-1}, k_{q}]$, $\bx_{\cD_{\infty}(i)[k]}$ is not $\cF_{k,i}$-measurable and $\eta_{\cD_{\infty}(i)[k]}$ is not $\cF_{k+1,i}$-measurable, which violate the assumption made in standard self-normalized bound \citep{abbasi2011improved,chowdhury2017kernelized,zenati2022efficient}. However, they become measurable if we condition on $\cF_{k_{q},i}$.
In addition, recall that in Section \ref{subsec:problem_formulation} we assume $\eta_{\cD_{\infty}(i)[k]}$ is zero-mean $R$-sub-Gaussian conditioned on $\sigma\bigl( (\bx_{s},\eta_{s})_{s \in \cN_{\cD_{\infty}(i)[k]-1}(i_{\cD_{\infty}(i)[k]})}, \bx_{\cD_{\infty}(i)[k]} \bigr)$, which is the $\sigma$-algebra generated by the sequence of local data collected by client $i_{\cD_{\infty}(i)[k]}$. We can see that as $\sigma\bigl( (\bx_{s},\eta_{s})_{s \in \cN_{\cD_{\infty}(i)[k]-1}(i_{\cD_{\infty}(i)[k]})}, \bx_{\cD_{\infty}(i)[k]} \bigr) \subseteq \cF_{k,i} $, $\eta_{\cD_{\infty}(i)[k]}$ is zero-mean $R$-sub-Gaussian conditioned on $\cF_{k,i}$. Basically, our assumption of $R$-sub-Gaussianity conditioned on \emph{local sequence} instead of \emph{global sequence} of data points, prevents the situation where the noise depends on data points that have not been communicated to the current client yet, i.e., they are not included in $\cF_{k,i}$.

With our `batched filtration' $\{\cF_{k,i}\}_{k \in \cB_{\infty}(i)}$ for each client $i$, we have everything we need to establish a time-uniform self-normalized bound that resembles 
Lemma 3.1 of \citep{zenati2022efficient}, but with improved logarithmic factor using the stopping time argument from \citep{abbasi2011improved}. Then we can take a union bound over $N$ clients to obtain the uniform bound over all clients and time steps.
\paragraph{Super-martingale \& self-normalized bound}
First, we need the following lemmas adopted from \citep{zenati2022efficient} to our `batched filtration'.
\begin{lemma} \label{lem:super_martingale}
Let $\upsilon \in \bR^{d}$ be arbitrary and consider for any $k \in \cB_{\infty}(i)$, $i \in [N]$, we have
\begin{align*}
    M_{k,i}^{\upsilon} = \exp\left( - \frac{1}{2} \lambda \upsilon^{\top}\upsilon + \sum_{s \in \cD_{\infty}(i)[1:k] } [ \frac{\eta_{s}<\upsilon, X_{s}>}{R} - \frac{1}{2} \upsilon^{\top} X_{s} X_{s}^{\top} \upsilon ]  \right)
\end{align*}
is a $\cF_{k+1,i}$-super-martingale, and $\bE[ M_{k,i}^{\upsilon} ] \leq \exp(-\frac{1}{2} \lambda \upsilon^{\top}\upsilon )$.
Let $\tilde{k}$ be a stopping time w.r.t. the filtration $\{\cF_{k,i}\}_{k \in \cB_{\infty}(i)}$. Then $M_{\tilde{k},i}^{\upsilon}$ is almost surely well-defined and $\bE[M_{\tilde{k},i}^{\upsilon}] \leq 1$.
\end{lemma}
\begin{proof}[Proof of Lemma \ref{lem:super_martingale}]
To show that $\{M_{k,i}^{\upsilon}\}_{ k \in \cB_{\infty}(i) }$ is a super-martinagle, we denote
\begin{align*}
    D_{k,i}^{\upsilon} = \exp\left( \frac{\eta_{\cD_{\infty}(i)[k]}<\upsilon, X_{\cD_{\infty}(i)[k]}>}{R} - \frac{1}{2} <\upsilon, X_{\cD_{\infty}(i)[k]}>^{2}  \right),
\end{align*}
with $D_{0,i}^{\upsilon} = \exp(-\frac{1}{2} \lambda \upsilon^{\top}\upsilon )$,
and as we have showed earlier, $\eta_{\cD_{\infty}(i)[k]}$ is $R$-sub-Gaussian conditioned on $\cF_{k,i}$. Therefore, $\bE[ D_{k,i}^{\upsilon} | \cF_{k,i} ] \leq 1$. Moreover, $D_{k,i}^{\upsilon}$ and $M_{k,i}^{\upsilon}$ are $\cF_{k+1,i}$-measurable. Then we have
\begin{align*}
    \bE[ M_{k,i}^{\upsilon} | \cF_{k,i} ] & = \bE[D_{1,i}^{\upsilon} D_{2,i}^{\upsilon} \dots D_{k-1,i}^{\upsilon} D_{k,i}^{\upsilon} | \cF_{k,i}]  \\
    & = D_{0,i}^{\upsilon} D_{1,i}^{\upsilon} D_{2,i}^{\upsilon} \dots D_{k-1,i}^{\upsilon} \bE[ D_{k,i}^{\upsilon} | \cF_{k,i}] \leq M_{k-1,i}^{\upsilon},
\end{align*}
which shows $\{M_{k,i}^{\upsilon}\}_{ k \in \cB_{\infty}(i) }$ is a super-martinagle, with $\bE[ M_{k,i}^{\upsilon} ] \leq D_{0,i}^{\upsilon}=\exp(-\frac{1}{2} \lambda \upsilon^{\top}\upsilon )$.
Then using the same argument as Lemma 8 of \citep{abbasi2011improved}, we have that $M_{\tilde{k},i}^{\upsilon}$ is almost surely well-defined, and $\bE[M_{\tilde{k},i}^{\upsilon}] \leq D_{0,i}^{\upsilon}=\exp(-\frac{1}{2} \lambda \upsilon^{\top}\upsilon )$.
\end{proof}

\begin{lemma} \label{lem:self_normalized_bound}
Let $\tilde{k}$ be a stopping time w.r.t. the filtration $\{\cF_{k,i}\}_{k \in \cB_{\infty}(i)}$. Then for $\delta > 0$, we have
\small
\begin{align*}
    & P\Big( \lVert  ( \lambda \bI +  \sum_{s \in \cD_{\infty}(i)[1:\tilde{k}] } X_{s} X_{s}^{\top} )^{-1/2} (\sum_{s \in \cD_{\infty}(i)[1:\tilde{k}] } X_{s} \eta_{s} ) \rVert > R\sqrt{2 \ln(1/\delta) + \ln( \det( \bK_{ \cD_{\infty}(i)[1:\tilde{k}], \cD_{\infty}(i)[1:\tilde{k}] }/\lambda + \bI) ) }\Big) \\
    & \leq \delta.
\end{align*}
\normalsize
\end{lemma}

\begin{proof}[Proof of Lemma \ref{lem:self_normalized_bound}]
Using $m:=( \lambda \bI +  \sum_{s \in \cD_{\infty}(i)[1:k] } X_{s} X_{s}^{\top} )^{-1} (\sum_{s \in \cD_{\infty}(i)[1:k] } X_{s} \eta_{s} )$, we can rewrite $M_{\tilde{k},i}^{\upsilon}$ as
\begin{align*}
    M_{\tilde{k},i}^{\upsilon} & = \exp\left( -\frac{1}{2} (\upsilon-m)^{\top} ( \lambda \bI +  \sum_{s \in \cD_{\infty}(i)[1:k] } X_{s} X_{s}^{\top} ) (\upsilon-m) \right) \\
    & \quad \times \exp\left( \frac{1}{2} \lVert  ( \lambda \bI +  \sum_{s \in \cD_{\infty}(i)[1:k] } X_{s} X_{s}^{\top} )^{-1/2} (\sum_{s \in \cD_{\infty}(i)[1:k] } X_{s} \eta_{s} ) \rVert^{2} \right).
\end{align*}
Then based on Lemma \ref{lem:super_martingale}, we have
\begin{align*}
    & \bE\Big[ \exp\Big( -\frac{1}{2} (\upsilon-m)^{\top} ( \lambda \bI +  \sum_{s \in \cD_{\infty}(i)[1:k] } X_{s} X_{s}^{\top} ) (\upsilon-m) \Big) \Big] \\
    & \quad + \bE\Big[\exp\Big( \frac{1}{2} \lVert  ( \lambda \bI +  \sum_{s \in \cD_{\infty}(i)[1:k] } X_{s} X_{s}^{\top} )^{-1/2} (\sum_{s \in \cD_{\infty}(i)[1:k] } X_{s} \eta_{s} ) \rVert^{2} \Big) \Big] \leq \exp(-\frac{1}{2} \lambda \upsilon^{\top}\upsilon )
\end{align*}
Following the Laplace's method as in proof of Lemma 3.1 of \citep{zenati2022efficient}, we have
\begin{align*}
    \bE\Big[\exp\Big( \frac{1}{2} \lVert  ( \lambda \bI +  \sum_{s \in \cD_{\infty}(i)[1:\tilde{k}] } X_{s} X_{s}^{\top} )^{-1/2} (\sum_{s \in \cD_{\infty}(i)[1:\tilde{k}] } X_{s} \eta_{s} ) \rVert^{2} \Big) \Big] \leq \sqrt{ \frac{\det( \bK_{ \cD_{\infty}(i)[1:\tilde{k}], \cD_{\infty}(i)[1:\tilde{k}] }+ \lambda \bI) }{\lambda^{\tilde{k}} }}
\end{align*}
By applying Markov-Chernov bound, we finish the proof.
\end{proof}


\paragraph{Proof of Lemma \ref{lem:confidence_ellipsoid_diskernelucb}}
Now using the stopping time argument as in \citep{abbasi2011improved}, and a union bound over clients, we can bound the second term in Eq~\eqref{eq:ellipsoid_intermediate}.
First, define the bad event
\begin{align*}
    B_{k}(\delta)  = \bigl\{ \omega \in \Omega: \lVert  ( \lambda \bI +  &\sum_{s \in \cD_{\infty}(i)[1:k] } X_{s} X_{s}^{\top} )^{-1/2} (\sum_{s \in \cD_{\infty}(i)[1:k] } X_{s} \eta_{s} ) \rVert > \\
    \quad \quad \quad \quad & R \sqrt{ 2 \ln(1/\delta)  + \ln( \det( \bK_{ \cD_{\infty}(i)[1:k], \cD_{\infty}(i)[1:k] }/\lambda + \bI) )} \bigr\},
\end{align*}
and $\tilde{k}(\omega)=\min( k \geq 0: \omega \in B_{k}(\delta) )$, which is a stopping time. Moreover, $\cup_{k \in  \cB_{\infty}(i)}B_{k}(\delta)=\{\omega: \tilde{k} < \infty\}$. Then using Lemma \ref{lem:self_normalized_bound}, we have
\small
\begin{align*}
    & P\bigl( \cup_{k \in  \cB_{\infty}(i)}B_{k}(\delta) \bigr) = P(\tilde{k} < \infty) \\
    & \leq P\Big( \lVert  ( \lambda \bI +  \sum_{s \in \cD_{\infty}(i)[1:\tilde{k}] } X_{s} X_{s}^{\top} )^{-1/2} (\sum_{s \in \cD_{\infty}(i)[1:\tilde{k}] } X_{s} \eta_{s} ) \rVert > R \sqrt{ 2 \ln(1/\delta) + \ln( \det( \bK_{ \cD_{\infty}(i)[1:\tilde{k}], \cD_{\infty}(i)[1:\tilde{k}] }/\lambda + \bI) )} \Big) \\
    & \leq \delta
\end{align*}
\normalsize
Note that $\cB_{\infty}(i)$ is the sequence of indices $k$ in $\cD_{\infty}(i)$ when client $i$ gets updated. Therefore, the result above is equivalent to
\begin{align*}
    \lVert \bA_{t,i}^{-1/2} \sum_{s \in \cD_{t}(i)} \phi_{s} \eta_{s} \rVert \leq R\sqrt{ 2 \ln(1/\delta) + \ln( \det( \bK_{ \cD_{t}(i), \cD_{t}(i) }/\lambda + \bI) ) }
\end{align*}
for all $t \geq 1$, with probability at least $1-\delta$. Then by taking a union bound over $N$ clients, we finish the proof.

\section{Proof of Lemma \ref{lem:regret_comm_diskernelucb}: Regret and Communication Cost of \algone{}} \label{sec:proof_regret_diskernelucb}
Based on Lemma \ref{lem:confidence_ellipsoid_diskernelucb} and the arm selection rule in Eq~\eqref{eq:UCB_exact}, we have
\begin{align*}
    f(\bx_{t}^{\star}) & \leq \hat{\mu}_{t-1,i_{t}}(\bx_{t}^{\star}) + \alpha_{t-1,i_{t}}\hat{\sigma}_{t-1,i_{t}}(\bx_{t}^{\star}) \leq \hat{\mu}_{t-1,i_{t}}(\bx_{t}) + \alpha_{t-1,i_{t}} \hat{\sigma}_{t-1,i_{t}}(\bx_{t}),  \\
    f(\bx_{t}) & \geq \hat{\mu}_{t-1,i_{t}}(\bx_{t}) - \alpha_{t-1,i_{t}} \hat{\sigma}_{t-1,i_{t}}(\bx_{t}),
\end{align*}
and thus $r_{t}=f(\bx_{t}^{\star})-f(\bx_{t}) \leq 2 \alpha_{t-1,i_{t}} \hat{\sigma}_{t-1,i_{t}}(\bx_{t})$, for all $t \in [NT]$, with probability at least $1-\delta$.
Then following similar steps as DisLinUCB of \citep{wang2019distributed}, we can obtain the regret and communication cost upper bound of \algone{}. 
\subsection{Regret Upper Bound} \label{subsec:regret_proof_diskernelucb}
We call the time period in-between two consecutive global synchronizations as an epoch, i.e., the $p$-th epoch refers to $[t_{p-1}+1,t_{p}]$, where $p \in [B]$ and $0 \leq B \leq NT$ denotes the total number of global synchronizations. 
Now consider an imaginary centralized agent that has immediate access to each data point in the learning system, and denote by $A_{t}=\sum_{s=1}^{t} \phi_{s} \phi_{s}^{\top}$ and $\bK_{[t],[t]}$ for $t \in [NT]$ the covariance matrix and kernel matrix constructed by this centralized agent.
Then similar to \citep{wang2019distributed}, we call the $p$-th epoch a good epoch if
\begin{align*}
    \ln\left(\frac{\det(\bI + \lambda^{-1}\bK_{[t_{p}],[t_{p}]})}{\det(\bI + \lambda^{-1}\bK_{[t_{p-1}],[t_{p-1}]})}\right) \leq 1,
\end{align*}
otherwise it is a bad epoch. 
Note that $\ln(\det(I + \lambda^{-1} \bK_{[NT],[NT]})) \leq 2 \gamma_{NT}$ by definition of $\gamma_{NT}$, i.e., the maximum information gain. 
Since $\ln(\frac{\det(\bI + \lambda^{-1}\bK_{[t_{1}],[t_{1}]})}{\det(\bI)})+\ln(\frac{\det(\bI + \lambda^{-1}\bK_{[t_{2}],[t_{2}]})}{\det(\bI + \lambda^{-1}\bK_{[t_{1}],[t_{1}]})})+\dots+\ln(\frac{\det(\bI + \lambda^{-1}\bK_{[NT],[NT]})}{\det(\bI + \lambda^{-1}\bK_{[t_{B}],[t_{B}]})}) = \ln(\det(I + \lambda^{-1} \bK_{[NT],[NT]})) \leq 2 \gamma_{NT}$, and due to the pigeonhole principle, there can be at most $2 \gamma_{NT}$ bad epochs.


If the instantaneous regret $r_{t}$ is incurred during a good epoch, we have
\begin{align*}
    r_{t} & \leq 2 \alpha_{t-1,i_{t}} \lVert \phi_{t} \rVert_{\bA_{t-1,i_{t}}^{-1}} \leq 2 \alpha_{t-1,i_{t}}\lVert \phi_{t} \rVert_{\bA_{t-1}^{-1}} \sqrt{\lVert \phi_{t} \rVert^{2}_{\bA_{t-1,i_{t}}^{-1}}/\lVert \phi_{t} \rVert^{2}_{\bA_{t-1}^{-1}}} \\
    & = 2 \alpha_{t-1,i_{t}}\lVert \phi_{t} \rVert_{\bA_{t-1}^{-1}} \sqrt{ \frac{\det(\bI + \lambda^{-1}\bK_{[t-1],[t-1]})}{\det(\bI + \lambda^{-1}\bK_{\cD_{t-1}(i_{t}),\cD_{t-1}(i_{t})})}  } \\
    & \leq 2 \sqrt{e} \alpha_{t-1,i_{t}}\lVert \phi_{t} \rVert_{\bA_{t-1}^{-1}} 
\end{align*}
where the second inequality is due to Lemma \ref{lem:quadratic_det_inequality_infinite}, and the last inequality is due to the definition of good epoch, i.e., $\frac{\det(\bI + \lambda^{-1}\bK_{[t-1],[t-1]})}{\det(\bI + \lambda^{-1}\bK_{\cD_{t-1}(i_{t}),\cD_{t-1}(i_{t})})} \leq \frac{\det(\bI + \lambda^{-1}\bK_{[t_{p}],[t_{p}]})}{\det(\bI + \lambda^{-1}\bK_{[t_{p-1}],[t_{p-1}]})} \leq e$.
Define $\alpha_{NT}:=\sqrt{\lambda} \lVert \theta_{\star} \rVert + \sqrt{ 2 \ln(N/\delta) + \ln( \det( \bK_{ [NT],[NT] }/\lambda + \bI) ) }$.
Then using standard arguments, the cumulative regret incurred in all good epochs can be bounded by,
\begin{align*}
    R_{good} & = \sum_{p=1}^{B} \mathbbm{1}\{\ln(\frac{\det(\bI + \lambda^{-1}\bK_{[t_{p}],[t_{p}]})}{\det(\bI + \lambda^{-1}\bK_{[t_{p-1}],[t_{p-1}]})}) \leq 1\} \sum_{t=t_{p-1}}^{t_{p}} r_{t} \leq \sum_{t=1}^{NT} 2 \sqrt{e} \alpha_{t-1,i_{t}} \lVert \phi_{t} \rVert_{\bA_{t-1}^{-1}} \\
    & \leq 2 \sqrt{e} \alpha_{NT} \sum_{t=1}^{NT} \lVert \phi_{t} \rVert_{\bA_{t-1}^{-1}} \leq 2 \sqrt{e} \alpha_{NT} \sqrt{ NT \cdot 2 \ln( \det(\bI + \lambda^{-1}\bK_{[NT],[NT]}) ) } \\
    & \leq 2 \sqrt{e} \alpha_{NT} \sqrt{ NT \cdot 4 \gamma_{NT} } =O\Big(\sqrt{NT}( \lVert \theta_{\star} \rVert \sqrt{\gamma_{NT}} + \gamma_{NT})\Big)
\end{align*}
where the third inequality is due to Cauchy-Schwartz and Lemma \ref{lem:sum_sqr_log_det}, and the forth is due to the definition of maximum information gain $\gamma_{NT}$.


Then we look at the regret incurred during bad epochs. Consider some bad epoch $p$, and the cumulative regret incurred during this epoch can be bounded by
\begin{align*}
    & \sum_{t=t_{p-1}+1}^{t_{p}} r_{t} = \sum_{i=1}^{N} \sum_{t \in \cD_{t_{p}}(i) \setminus \cD_{t_{p-1}}(i) }  r_{t} \leq  2 \alpha_{NT} \sum_{i=1}^{N} \sum_{t \in \cD_{t_{p}}(i) \setminus \cD_{t_{p-1}}(i) } \lVert \phi_{t} \rVert_{\bA_{t-1,i}^{-1}} \\
    & \leq 2 \alpha_{NT} \sum_{i=1}^{N} \sqrt{ (| \cD_{t_{p}}(i) | - | \cD_{t_{p-1}}(i) |) \sum_{t \in \cD_{t_{p}}(i) \setminus \cD_{t_{p-1}}(i) } \lVert \phi_{t} \rVert_{\bA_{t-1,i}^{-1}}^{2} }  \\
    & \leq 2  \alpha_{NT} \sum_{i=1}^{N} \sqrt{ 2(| \cD_{t_{p}}(i) | - | \cD_{t_{p-1}}(i) |) \ln(\frac{\det(\bI + \lambda^{-1} \bK_{\cD_{t_{p}}(i),\cD_{t_{p}}(i)})}{\det(\bI + \lambda^{-1} K_{\cD_{t_{p-1}}(i), \cD_{t_{p-1}}(i)})})  } \\
    & \leq 2\sqrt{2}  \alpha_{NT} N \sqrt{D}
\end{align*}
where the last inequality is due to our event-trigger in Eq~\eqref{eq:sync_event_exact}.
Since there can be at most $2 \gamma_{NT}$ bad epochs, the cumulative regret incurred in all bad epochs
\begin{align*}
    R_{bad} & \leq 2 \gamma_{NT} \cdot 2\sqrt{2} \alpha_{NT} N \sqrt{ D} = O\Big(N D^{0.5} ( \lVert \theta_{\star} \rVert \gamma_{NT} + \gamma_{NT}^{1.5})\Big)
\end{align*}
Combining cumulative regret incurred during both good and bad epochs, we have
\begin{align*}
    R_{NT} = R_{good} + R_{bad} = O\bigl( (\sqrt{NT} + N \sqrt{D \gamma_{NT}}) ( \lVert \theta_{\star} \rVert \sqrt{\gamma_{NT}} + \gamma_{NT}) \bigr)
\end{align*}


\subsection{Communication Upper Bound}
For some $\alpha > 0$, there can be at most $\lceil \frac{NT}{\alpha} \rceil$ epochs with length larger than $\alpha$. 
Based on our event-trigger design, we know that $(| \cD_{t_{p}}(i_{t_{p}}) | - | \cD_{t_{p-1}}(i_{t_{p}}) |) \ln(\frac{\det(\bI + \lambda^{-1} \bK_{[t_{p}],[t_{p}]})}{\det(\bI + \lambda^{-1} K_{[t_{p-1}],[t_{p-1}]})}) \geq (| \cD_{t_{p}}(i_{t_{p}}) | - | \cD_{t_{p-1}}(i_{t_{p}}) |) \ln(\frac{\det(\bI + \lambda^{-1} \bK_{\cD_{t_{p}}(i_{t_{p}}),\cD_{t_{p}}(i_{t_{p}})})}{\det(\bI + \lambda^{-1} K_{\cD_{t_{p-1}}(i_{t_{p}}), \cD_{t_{p-1}}(i_{t_{p}})})}) \geq D$ for any epoch $p \in [B]$, where $i_{t_{p}}$ is the client who triggers the global synchronization at time step $t_{p}$.
Then if the length of certain epoch $p$ is smaller than $\alpha$, i.e., $t_{p}-t_{p-1} \leq \alpha$, we have $\ln(\frac{\det(\bI + \lambda^{-1} \bK_{[t_{p}],[t_{p}]})}{\det(\bI + \lambda^{-1} K_{[t_{p-1}],[t_{p-1}]})}) \geq \frac{N D}{\alpha}$. Since $\ln(\frac{\det(\bI + \lambda^{-1} \bK_{[t_{1}],[t_{1}]})}{\det(\bI)}) + \ln(\frac{\det(\bI + \lambda^{-1} \bK_{[t_{2}], [t_{2}]})}{\det(\bI + \lambda^{-1} \bK_{[t_{1}],[t_{1}]})}) + \dots + \ln(\frac{\det(\bI + \lambda^{-1} \bK_{[t_{B}],[t_{B}]})}{\det(\bI + \lambda^{-1} \bK_{[t_{B-1}], [t_{B-1}]})}) \leq \ln(\det(\bI + \lambda^{-1} \bK_{[NT], [NT]})) \leq 2 \gamma_{NT}$, the total number of such epochs is upper bounded by $\lceil \frac{2\gamma_{NT} \alpha}{ND} \rceil$. Combining the two terms, the total number of epochs can be bounded by,
\begin{align*}
    B \leq \lceil \frac{NT}{\alpha} \rceil + \lceil \frac{2\gamma_{NT} \alpha}{ND} \rceil 
\end{align*}
where the LHS can be minimized using the AM-GM inequality, i.e., $B \leq \sqrt{ \frac{NT}{\alpha} \frac{2\gamma_{NT} \alpha}{ND} }= \sqrt{\frac{2 \gamma_{NT} T}{D}}$.
To obtain the optimal order of regret, we set $D=O(\frac{T}{N \gamma_{NT}})$, so that $R_{NT}=O\bigl( \sqrt{NT}  ( \lVert \theta_{\star} \rVert \sqrt{\gamma_{NT}} + \gamma_{NT}) \bigr)$. And the total number of epochs $B = O(\sqrt{N} \gamma_{NT})$. However, we should note that as \algone{} communicates all the unshared raw data at each global synchronization, the total communication cost mainly depends on when the last global synchronization happens. Since the sequence of candidate sets $\{\cA_{t}\}_{t \in [NT]}$, which controls the growth of determinant, is an arbitrary subset of $\cA$, the time of last global synchronization could happen at the last time step $t=NT$.
Therefore, $C_{T}=O(N^{2}T d)$ in such a worst case.


\section{Derivation of the Approximated Mean and Variance in Section \ref{sec:method}} \label{sec:derivation_approx}
For simplicity, subscript $t$ is omitted in this section.
The approximated Ridge regression estimator for dataset $\{(\bx_{s},y_{s})\}_{s \in \cD}$ is formulated as
\begin{align*}
    \tilde{\btheta}=\argmin_{\btheta \in \cH} \sum_{s \in \cD} \Big( (\bP_{\cS}\phi_{s})^{\top} \btheta - y_{s}\Big)^{2} + \lambda \lVert \btheta \rVert_{2}^{2}
\end{align*}
where $\cD$ denotes the sequence of time indices for data in the original dataset, $\cS \subseteq \cD$ denotes the time indices for data in the dictionary, and $\bP_{\cS} \in \bR^{p \times p}$ denotes the orthogonal projection matrix defined by $\cS$.
Then by taking derivative and setting it to zero, we have $(\bP_{\cS} \bPhi_{\cD}^{\top} \bPhi_{\cD} \bP_{\cS} + \lambda \bI) \tilde{\btheta}=\bP_{\cS} \bPhi_{\cD}^{\top} \by_{\cD}$, and thus $\tilde{\theta}=\tilde{\bA}^{-1} \tilde{\bb}$, where $\tilde{\bA}=\bP_{\cS} \bPhi_{\cD}^{\top} \bPhi_{\cD} \bP_{\cS} + \lambda \bI$ and $\tilde{\bb}=\bP_{\cS} \bPhi_{\cD}^{\top} \by_{\cD}$.

Hence, the approximated mean reward and variance for some arm $\bx$ are
\begin{align*}
    \tilde{\mu}_{t,i}(\bx) &= \phi(\bx)^{\top} \tilde{\bA}^{-1} \tilde{\bb} \\
    \tilde{\sigma}_{t,i}(\bx) &= \sqrt{ \phi(\bx)^{\top} \tilde{\bA}^{-1} \phi(\bx)}
\end{align*}
To obtain their kernelized representation, we rewrite
\begin{align*}
    &(\bP_{\cS} \bPhi_{\cD}^{\top} \bPhi_{\cD} \bP_{\cS} + \lambda \bI) \tilde{\btheta}=\bP_{\cS} \bPhi_{\cD}^{\top} \by_{\cD} \\ \Leftrightarrow ~ & \bP_{\cS} \bPhi_{\cD}^{\top} (\by_{\cD} - \bPhi_{\cD} \bP_{\cS} \tilde{\btheta}) = \lambda \tilde{\btheta} \\
    \Leftrightarrow ~ & \tilde{\btheta} = \bP_{\cS} \bPhi_{\cD}^{\top} \rho
\end{align*}
where $\rho := \frac{1}{\lambda}(\by_{\cD} - \bPhi_{\cD} \bP_{\cS} \tilde{\btheta})=\frac{1}{\lambda}(\by_{\cD} - \bPhi_{\cD} \bP_{\cS} \bP_{\cS} \bPhi_{\cD}^{\top} \rho)$. Solving this equation, we get $\rho=(\bPhi_{\cD} \bP_{\cS} \bP_{\cS} \bPhi_{\cD}^{\top} + \lambda \bI)^{-1} \by_{\cD}$. 
Note that $\bP_{\cS}\bP_{\cS} = \bP_{\cS}$, since projection matrix $\bP_{\cS}$ is idempotent. Moreover, we have $( \bPhi^{\top}\bPhi + \lambda \bI)\bPhi^{\top}=\bPhi^{\top}(\bPhi \bPhi^{\top} + \lambda \bI)$, and $( \bPhi^{\top}\bPhi + \lambda \bI)^{-1}\bPhi^{\top}=\bPhi^{\top}(\bPhi \bPhi^{\top} + \lambda \bI)^{-1}$.
Therefore, we can rewrite the approximated mean for some arm $\bx$ as
\begin{align*}
     \tilde{\mu}(\bx) & = \phi(\bx)^{\top} \bP_{\cS} \bPhi_{\cD}^{\top} (\bPhi_{\cD} \bP_{\cS} \bP_{\cS} \bPhi_{\cD}^{\top} + \lambda \bI)^{-1} \by_{\cD} \\
    & = (\bP_{\cS}^{1/2} \phi(\bx))^{\top}  (\bPhi_{\cD}\bP_{\cS}^{1/2})^{\top}[\bPhi_{\cD} \bP_{\cS}^{1/2} (\bPhi_{\cD} \bP_{\cS}^{1/2})^{\top} + \lambda \bI]^{-1} \by_{\cD} \\
    & = (\bP_{\cS}^{1/2} \phi(\bx))^{\top}  (\bP_{\cS}^{1/2}\bPhi_{\cD}^{\top}\bPhi_{\cD}\bP_{\cS}^{1/2} + \lambda \bI)^{-1} (\bPhi_{\cD}\bP_{\cS}^{1/2})^{\top}\by_{\cD} \\
    & = z(\bx;\cS)^{\top} \bigl( \bZ_{\cD;\cS}^{\top}\bZ_{\cD;\cS} + \lambda \bI\bigr)^{-1} \bZ_{\cD;\cS}^{\top} \by_{\cD}
\end{align*}
To derive the approximated variance, we start from the fact that $(\bP_{\cS} \bPhi_{\cD}^{\top} \bPhi_{\cD}\bP_{\cS} +\lambda \bI) \phi(\bx)=\bP_{\cS} \bPhi_{\cD}^{\top} \bPhi_{\cD}\bP_{\cS}\phi(\bx) + \lambda \phi(\bx)$, so
\begin{align*}
     \phi(\bx)&=(\bP_{\cS} \bPhi_{\cD}^{\top} \bPhi_{\cD}\bP_{\cS} +\lambda \bI)^{-1} \bP_{\cS} \bPhi_{\cD}^{\top} \bPhi_{\cD}\bP_{\cS}\phi(\bx) + \lambda (\bP_{\cS} \bPhi_{\cD}^{\top} \bPhi_{\cD}\bP_{\cS} +\lambda \bI)^{-1}\phi(\bx)\\
    & =\bP_{\cS} \bPhi_{\cD}^{\top} (\bPhi_{\cD}\bP_{\cS} \bP_{\cS}  \bPhi_{\cD}^{\top}+\lambda \bI)^{-1} \bPhi_{\cD}\bP_{\cS} \phi(\bx) + \lambda (\bP_{\cS} \bPhi_{\cD}^{\top} \bPhi_{\cD}\bP_{\cS} +\lambda \bI)^{-1}\phi(\bx)
\end{align*}
Then we have
\begin{align*}
    & \phi(\bx)^{\top} \phi(\bx) \\
    = & \bigl\{ \bP_{\cS} \bPhi_{\cD}^{\top} (\bPhi_{\cD}\bP_{\cS} \bP_{\cS}  \bPhi_{\cD}^{\top}+\lambda \bI)^{-1} \bPhi_{\cD}\bP_{\cS} \phi(\bx) + \lambda (\bP_{\cS} \bPhi_{\cD}^{\top} \bPhi_{\cD}\bP_{\cS} +\lambda \bI)^{-1}\phi(\bx) \bigr\}^{\top} \\
    & \quad \quad \bigl\{ \bP_{\cS} \bPhi_{\cD}^{\top} (\bPhi_{\cD}\bP_{\cS} \bP_{\cS}  \bPhi_{\cD}^{\top}+\lambda \bI)^{-1} \bPhi_{\cD}\bP_{\cS} \phi(\bx) + \lambda (\bP_{\cS} \bPhi_{\cD}^{\top} \bPhi_{\cD}\bP_{\cS} +\lambda \bI)^{-1}\phi(\bx) \bigr\} \\
 = & \phi(\bx)^{\top} \bP_{\cS} \bPhi_{\cD}^{\top} (\bPhi_{\cD}\bP_{\cS} \bP_{\cS}  \bPhi_{\cD}^{\top}+\lambda \bI)^{-1} \bPhi_{\cD} \bP_{\cS}\bP_{\cS} \bPhi_{\cD}^{\top} (\bPhi_{\cD}\bP_{\cS} \bP_{\cS}  \bPhi_{\cD}^{\top}+\lambda \bI)^{-1} \bPhi_{\cD}\bP_{\cS} \phi(\bx) \\
    &  + 2 \lambda \phi(\bx)^{\top} \bP_{\cS} \bPhi_{\cD}^{\top} (\bPhi_{\cD}\bP_{\cS} \bP_{\cS}  \bPhi_{\cD}^{\top}+\lambda \bI)^{-1} \bPhi_{\cD} \bP_{\cS} (\bP_{\cS} \bPhi_{\cD}^{\top} \bPhi_{\cD}\bP_{\cS} +\lambda \bI)^{-1}\phi(\bx) \\
    &  + \lambda \phi(\bx)^{\top} (\bP_{\cS} \bPhi_{\cD}^{\top} \bPhi_{\cD}\bP_{\cS} +\lambda \bI)^{-1} \lambda \bI (\bP_{\cS} \bPhi_{\cD}^{\top} \bPhi_{\cD}\bP_{\cS} +\lambda \bI)^{-1} \phi(\bx) \\
    = & \phi(\bx)^{\top} \bP_{\cS} \bPhi_{\cD}^{\top} (\bPhi_{\cD}\bP_{\cS} \bP_{\cS}  \bPhi_{\cD}^{\top}+\lambda \bI)^{-1} \bPhi_{\cD}\bP_{\cS} \phi(\bx) + \lambda \phi(\bx)^{\top} (\bP_{\cS} \bPhi_{\cD}^{\top} \bPhi_{\cD}\bP_{\cS} +\lambda \bI)^{-1} \phi(\bx)
\end{align*}
By rearranging terms, we have
\begin{align*}
    \tilde{\sigma}^{2}(\bx) = & \phi(\bx)^{\top} (\bP_{\cS} \bPhi_{\cD}^{\top} \bPhi_{\cD}\bP_{\cS} +\lambda \bI)^{-1} \phi(\bx) \\
     = & \frac{1}{\lambda} \bigl\{ \phi(\bx)^{\top} \phi(\bx) - \phi(\bx)^{\top} \bP_{\cS} \bPhi_{\cD}^{\top} (\bPhi_{\cD}\bP_{\cS} \bP_{\cS}  \bPhi_{\cD}^{\top}+\lambda \bI)^{-1} \bPhi_{\cD}\bP_{\cS} \phi(\bx) \bigr\} \\
 = & \frac{1}{\lambda} \{ k(\bx, \bx) -  z(\bx;\cS)^{\top} \bZ_{\cD;\cS}^{\top}\bZ_{\cD;\cS} [ \bZ_{\cD;\cS}^{\top}\bZ_{\cD;\cS} + \lambda \bI]^{-1} z(\bx|\cS)  \} 
\end{align*}


\section{Proof of Lemma \ref{lem:dictionary_accuracy_global}}
In the following, we analyze the $\epsilon_{t,i}$-accuracy of the dictionary for all $t,i$.

At the time steps when global synchronization happens, i.e., $t_{p}$ for $p \in [B]$, $\cS_{t_{p}}$ is sampled from $[t_{p}]=\cD_{t_{p}}(i)$ using approximated variance $\tilde{\sigma}^{2}_{t_{p-1},i}$. In this case, the accuracy of the dictionary only depends on the RLS procedure, and 
Calandriello et al. \citep{calandriello2020near} have already showed that the following guarantee on the accuracy and size of dictionary holds $\forall t \in \{t_{p}\}_{p \in [B]}$.
\begin{lemma}[Lemma 2 of \citep{calandriello2020near}] \label{lem:rls}
Under the condition that $\bar{q}=6\frac{1+\epsilon}{1-\epsilon} \log(4NT/\delta)/\epsilon^{2}$, for some $\epsilon \in [0,1)$, with probability at least $1-\delta$, we have $\forall t \in \{t_{p}\}_{p \in [B]}$ that the dictionary $\{ (\bx_{s},y_{s}) \}_{s \in \cS_{t}}$ is $\epsilon$-accurate w.r.t. $\{ (\bx_{s},y_{s}) \}_{s \in \cD_{t}(i)}$, and $\frac{1-\epsilon}{1+\epsilon} \sigma^{2}_{t}(\bx)  \leq \tilde{\sigma}^{2}_{t}(\bx) \leq \frac{1+\epsilon}{1-\epsilon} \sigma^{2}_{t}(\bx) , \forall \bx \in \cA$. Moreover, the size of dictionary $|\cS_{t}| \leq 3 (1+L^{2}/\lambda) \frac{1+\epsilon}{1-\epsilon} \bar{q} \tilde{d}$, where $\tilde{d}:=\text{Tr}(\bK_{[NT],[NT]} (\bK_{[NT],[NT]} + \lambda \bI)^{-1})$ denotes the effective dimension of the problem, and it is known that $\tilde{d}=O(\gamma_{NT})$ \citep{chowdhury2017kernelized}.
\end{lemma}
Lemma \ref{lem:rls} guarantees that for all $t \in \{t_{p}\}_{p \in [B]}$, the dictionary has a constant accuracy, i.e., $\epsilon_{t,i}=\epsilon,\forall i$. In addition, since the dictionary is fixed for $t \notin \{t_{p}\}_{p \in [B]}$, its size $\cS_{t} = O(\gamma_{NT}), \forall t \in [NT]$.

Then for time steps $t \notin \{t_{p}\}_{p \in [B]}$, due to the local update, the accuracy of the dictionary will degrade. However, thanks to our event-trigger in Eq~\eqref{eq:sync_event}, the extent of such degradation can be controlled, i.e., a new dictionary update will be triggered before the previous dictionary becomes completely irrelevant. This is shown in Lemma \ref{lem:dictionary_accuracy} below.
\begin{lemma} \label{lem:dictionary_accuracy}
Under the condition that $\{ (\bx_{s},y_{s}) \}_{s \in \cS_{t_{p}}}$ is $\epsilon$-accurate w.r.t. $\{ (\bx_{s},y_{s}) \}_{s \in \cD_{t_{p}}(i)}$, $\forall t \in [t_{p}+1, t_{p+1}], i\in[N]$,
$\cS_{t_{p}}$ is $\bigl(\epsilon+1- \frac{1}{1+\frac{1+\epsilon}{1-\epsilon}D} \bigr)$-accurate w.r.t. $\cD_{t}(i)$.
\end{lemma}
Combining Lemma \ref{lem:rls} and Lemma \ref{lem:dictionary_accuracy} finishes the proof.


\begin{proof}[Proof of Lemma \ref{lem:dictionary_accuracy}]
Similar to \citep{calandriello2019gaussian}, we can rewrite the $\epsilon$-accuracy condition of $\cS_{t_{p}}$ w.r.t. $\cD_{t}(i)$ for $t \in [t_{p}+1, t_{p+1}]$ as
\begin{align*}
    & (1-\epsilon_{t,i}) (\bPhi_{\cD_{t}(i)}^{\top} \bPhi_{\cD_{t}(i)} + \lambda \bI) \preceq \bPhi_{\cD_{t}(i)}^{\top} \bar{\bS}_{t,i}^{\top} \bar{\bS}_{t,i}\bPhi_{\cD_{t}(i)} + \lambda \bI \preceq (1+\epsilon_{t,i}) (\bPhi_{\cD_{t}(i)}^{\top} \bPhi_{\cD_{t}(i)} + \lambda \bI) \\
    \Leftrightarrow & -\epsilon_{t,i} (\bPhi_{\cD_{t}(i)}^{\top} \bPhi_{\cD_{t}(i)} + \lambda \bI) \preceq \bPhi_{\cD_{t}(i)}^{\top} \bar{\bS}_{t,i}^{\top} \bar{\bS}_{t,i}\bPhi_{\cD_{t}(i)} - \bPhi_{\cD_{t}(i)}^{\top} \bPhi_{\cD_{t}(i)} \preceq \epsilon_{t,i} (\bPhi_{\cD_{t}(i)}^{\top} \bPhi_{\cD_{t}(i)} + \lambda \bI) \\
    \Leftrightarrow  & - \epsilon_{t,i} \bI \preceq (\bPhi_{\cD_{t}(i)}^{\top} \bPhi_{\cD_{t}(i)} + \lambda \bI)^{-1/2} (\bPhi_{\cD_{t}(i)}^{\top} \bar{\bS}_{t,i}^{\top} \bar{\bS}_{t,i}\bPhi_{\cD_{t}(i)} - \bPhi_{\cD_{t}(i)}^{\top} \bPhi_{\cD_{t}(i)}) (\bPhi_{\cD_{t}(i)}^{\top} \bPhi_{\cD_{t}(i)} + \lambda \bI)^{-1/2} \preceq \epsilon_{t,i} \bI \\
    \Leftrightarrow  & \lVert (\bPhi_{\cD_{t}(i)}^{\top} \bPhi_{\cD_{t}(i)} + \lambda \bI)^{-1/2} (\bPhi_{\cD_{t}(i)}^{\top} \bar{\bS}_{t,i}^{\top} \bar{\bS}_{t,i}\bPhi_{\cD_{t}(i)} - \bPhi_{\cD_{t}(i)}^{\top} \bPhi_{\cD_{t}(i)}) (\bPhi_{\cD_{t}(i)}^{\top} \bPhi_{\cD_{t}(i)} + \lambda \bI)^{-1/2} \rVert \leq \epsilon_{t,i}  \\
    \Leftrightarrow  & \lVert \sum_{s \in \cD_{t_{p}} } (\frac{q_{s}}{\tilde{p}_{s}}-1) \psi_{s} \psi_{s}^{\top} + \sum_{s \in \cD_{t}(i) \setminus \cD_{t_{p}} } (0-1) \psi_{s} \psi_{s}^{\top} \rVert \leq \epsilon_{t,i}
\end{align*}
where $\psi_{s}=(\bPhi_{\cD_{t}(i)}^{\top} \bPhi_{\cD_{t}(i)} + \lambda \bI)^{-1/2} \phi_{s}$. Notice that the second term in the norm has weight $-1$ because the dictionary $\cS_{t_{p}}$ is fixed after $t_{p}$. With triangle inequality, now it suffices to bound 
\begin{align*}
    \lVert \sum_{s \in \cD_{t_{p}} } (\frac{q_{s}}{\tilde{p}_{s}}-1) \psi_{s,j} \psi_{s}^{\top} + \sum_{s \in \cD_{t}(i) \setminus \cD_{t_{p}} } (0-1) \psi_{s} \psi_{s}^{\top} \rVert  \leq \lVert \sum_{s \in \cD_{t_{p}} } (\frac{q_{s}}{\tilde{p}_{s}}-1) \psi_{s} \psi_{s}^{\top} \rVert + \lVert \sum_{s \in \cD_{t}(i) \setminus \cD_{t_{p}} } \psi_{s} \psi_{s}^{\top} \rVert.
\end{align*}
We should note that the first term corresponds to the approximation accuracy of $\cS_{t_{p}}$ w.r.t. the dataset $\cD_{t_{p}}$. And under the condition that it is $\epsilon$-accurate w.r.t. $\cD_{t_{p}}$, we have $\lVert \sum_{s \in \cD_{t_{p}} } (\frac{q_{s}}{\tilde{p}_{s}}-1) \psi_{s} \psi_{s}^{\top} \rVert \leq \epsilon$. The second term measures the difference between $\cD_{t}(i)$ compared with $\cD_{t_{p}}$, which is unique to our work. We can bound it as follows.
\begin{align*}
    & \lVert \sum_{s \in \cD_{t}(i) \setminus \cD_{t_{p}} } \psi_{s} \psi_{s}^{\top} \rVert \\
    = & \lVert  (\bPhi_{\cD_{t}(i)}^{\top} \bPhi_{\cD_{t}(i)} + \lambda \bI)^{-1/2} (\sum_{s \in \cD_{t}(i) \setminus \cD_{t_{p}} } \phi_{s} \phi_{s}^{\top}) (\bPhi_{\cD_{t}(i)}^{\top} \bPhi_{\cD_{t}(i)} + \lambda \bI)^{-1/2} \rVert \\
    = & \max_{\phi \in \cH} \frac{\phi^{\top} (\bPhi_{\cD_{t}(i)}^{\top} \bPhi_{\cD_{t}(i)} + \lambda \bI)^{-1/2} (\sum_{s \in \cD_{t}(i) \setminus \cD_{t_{p}} } \phi_{s} \phi_{s}^{\top}) (\bPhi_{\cD_{t}(i)}^{\top} \bPhi_{\cD_{t}(i)} + \lambda \bI)^{-1/2} \phi}{\phi^{\top} \phi} \\
    = & \max_{\phi \in \cH} \frac{\phi^{\top} (\sum_{s \in \cD_{t}(i) \setminus \cD_{t_{p}} } \phi_{s} \phi_{s}^{\top}) \phi}{\phi^{\top} (\bPhi_{\cD_{t}(i)}^{\top} \bPhi_{\cD_{t}(i)} + \lambda \bI) \phi} \\
    = & 1 - \min_{\phi \in \cH} \frac{\phi^{\top} (\bPhi_{\cD_{t_{p}}}^{\top}\bPhi_{\cD_{t_{p}}} + \lambda \bI) \phi}{\phi^{\top} (\bPhi_{\cD_{t}(i)}^{\top}\bPhi_{\cD_{t}(i)} + \lambda \bI) \phi} \\
    = & 1 - \frac{1}{\max_{\phi \in \cH} \frac{\phi^{\top} (\bPhi_{\cD_{t_{p}}}^{\top}\bPhi_{\cD_{t_{p}}} + \lambda \bI)^{-1} \phi}{\phi^{\top} (\bPhi_{\cD_{t}(i)}^{\top}\bPhi_{\cD_{t}(i)} + \lambda \bI)^{-1} \phi}} \\
    = & 1- \frac{1}{\max_{\bx} \frac{\sigma^{2}_{t_{p},i}(\bx)}{\sigma^{2}_{t,i}(\bx)}}
\end{align*}
We can further bound the term $\frac{\sigma^{2}_{t_{p},i}(\bx)}{\sigma^{2}_{t,i}(\bx)}$ using the threshold of the event-trigger in Eq~\eqref{eq:sync_event}. For any $\bx \in \bR^{d}$,
\begin{align*}
    \frac{\sigma^{2}_{t_{p},i}(\bx)}{\sigma^{2}_{t,i}(\bx)} \leq 1+\sum_{s \in \cD_{t}(i) \setminus \cD_{t_{p}}} \hat{\sigma}_{t_{p},i}^{2}(\bx_{s})  \leq 1 + \frac{1+\epsilon}{1-\epsilon} \sum_{s \in \cD_{t}(i) \setminus \cD_{t_{p}}} \tilde{\sigma}_{t_{p},i}^{2}(\bx_{s}) \leq 1 + \frac{1+\epsilon}{1-\epsilon} D
\end{align*}
where the first inequality is due to Lemma \ref{lem:bound_variance_ratio}, the second is due to Lemma \ref{lem:rls}, and the third is due to the event-trigger in Eq~\eqref{eq:sync_event}.
Putting everything together, we have that if $\cS_{t_{p}}$ is $\epsilon$-accurate w.r.t. $\cD_{t_{p}}$, then it is $\bigl(\epsilon+1- \frac{1}{1 + \frac{1+\epsilon}{1-\epsilon} D} \bigr)$-accurate w.r.t. dataset $\cD_{t}(i)$, which finishes the proof.
\end{proof}

\section{Proof of Lemma \ref{lem:confidence_ellipsoid_approx}} \label{sec:proof_confidence_ellipsoid_approx}
To prove Lemma \ref{lem:confidence_ellipsoid_approx}, we need the following lemma.
\begin{lemma} \label{lem:confidence_ellipsoid_intermediate}
We have $\forall t,i$ that
\begin{align*}
\small
    \lVert\tilde{\btheta}_{t,i} - \btheta_{\star} \rVert_{\tilde{\bA}_{t,i}} \leq \Big( \lVert \bPhi_{\cD_{t}(i)} (\bI - \bP_{\cS}) \rVert   + \sqrt{\lambda}  \Big) \lVert \btheta_{\star} \rVert + R \sqrt{4 \ln{N/\delta}+ 2\ln{\det((1+\lambda) \bI + \bK_{\cD_{t}(i), \cD_{t}(i)})}}
\end{align*}
\normalsize
with probability at least $1-\delta$.
\end{lemma}
\begin{proof}[Proof of Lemma \ref{lem:confidence_ellipsoid_intermediate}]
Recall that the approximated kernel Ridge regression estimator for $\theta_{\star}$ is defined as
\begin{align*}
    \tilde{\btheta}_{t,i} = \tilde{\bA}_{t,i}^{-1} \bP_{\cS} \bPhi_{\cD_{t}(i)}^{\top} \by_{\cD_{t}(i)}
\end{align*}
where $\bP_{\cS}$ is the orthogonal projection matrix for the Nystr\"{o}m approximation, and $\tilde{\bA}_{t,i}=\bP_{\cS} \bPhi_{\cD_{t}(i)}^{\top}\bPhi_{\cD_{t}(i)} \bP_{\cS} + \lambda \bI$. Then our goal is to bound
\begin{align*}
    & (\tilde{\btheta}_{t,i} - \btheta_{\star})^{\top} \tilde{\bA}_{t,i} (\tilde{\btheta}_{t,i} - \btheta_{\star}) \\ = & (\tilde{\btheta}_{t,i} - \btheta_{\star})^{\top} \tilde{\bA}_{t,i} (\tilde{\bA}_{t,i}^{-1}\bP_{\cS} \bPhi_{\cD_{t}(i)}^{\top} \by_{\cD_{t}(i)} - \btheta_{\star})\\
    = &(\tilde{\btheta}_{t,i} - \btheta_{\star})^{\top} \tilde{\bA}_{t,i} [\tilde{\bA}_{t,i}^{-1}\bP_{\cS} \bPhi_{\cD_{t}(i)}^{\top} (\bPhi_{\cD_{t}(i)} \btheta_{\star} + \mathbf{\eta}_{\cD_{t}(i)}) - \btheta_{\star}] \\
    = &(\tilde{\btheta}_{t,i} - \btheta_{\star})^{\top} \tilde{\bA}_{t,i} (\tilde{\bA}_{t,i}^{-1}\bP_{\cS} \bPhi_{\cD_{t}(i)}^{\top} \bPhi_{\cD_{t}(i)} \btheta_{\star}  - \btheta_{\star}) + (\tilde{\btheta}_{t,i} - \btheta_{\star})^{\top} \tilde{\bA}_{t,i} \tilde{\bA}_{t,i}^{-1}\bP_{\cS} \bPhi_{\cD_{t}(i)}^{\top} \mathbf{\eta}_{\cD_{t}(i)}
\end{align*}
\paragraph{Bounding the first term}
To bound the first term, we begin with rewriting
\begin{align*}
    & \tilde{\bA}_{t,i} (\tilde{\bA}_{t,i}^{-1}\bP_{\cS} \bPhi_{\cD_{t}(i)}^{\top} \bPhi_{\cD_{t}(i)} \btheta_{\star}  - \btheta_{\star}) \\
    = & \bP_{\cS} \bPhi_{\cD_{t}(i)}^{\top} \bPhi_{\cD_{t}(i)} \btheta_{\star} - \bP_{\cS} \bPhi_{\cD_{t}(i)}^{\top}\bPhi_{\cD_{t}(i)} \bP_{\cS} \btheta_{\star} - \lambda \btheta_{\star} \\
     = & \bP_{\cS} \bPhi_{\cD_{t}(i)}^{\top} \bPhi_{\cD_{t}(i)} (\bI - \bP_{\cS}) \btheta_{\star} - \lambda \btheta_{\star} 
\end{align*}
and by substituting this into the first term, we have
\begin{align*}
    & (\tilde{\btheta}_{t,i} - \btheta_{\star})^{\top} \tilde{\bA}_{t,i} (\tilde{\bA}_{t,i}^{-1}\bP_{\cS} \bPhi_{\cD_{t}(i)}^{\top} \bPhi_{\cD_{t}(i)} \btheta_{\star}  - \btheta_{\star}) \\
    = & (\tilde{\btheta}_{t,i} - \btheta_{\star})^{\top}\bP_{\cS} \bPhi_{\cD_{t}(i)}^{\top} \bPhi_{\cD_{t}(i)} (\bI - \bP_{\cS}) \btheta_{\star} - \lambda (\tilde{\btheta}_{t,i} - \btheta_{\star})^{\top} \btheta_{\star} \\
    =& (\tilde{\btheta}_{t,i} - \btheta_{\star})^{\top} \tilde{\bA}_{t,i}^{1/2} \tilde{\bA}_{t,i}^{-1/2} \bP_{\cS} \bPhi_{\cD_{t}(i)}^{\top} \bPhi_{\cD_{t}(i)} (\bI - \bP_{\cS}) \btheta_{\star} - \lambda (\tilde{\btheta}_{t,i} - \btheta_{\star})^{\top} \tilde{\bA}_{t,i}^{1/2} \tilde{\bA}_{t,i}^{-1/2} \btheta_{\star} \\
    \leq & \lVert \tilde{\btheta}_{t,i} - \btheta_{\star} \rVert_{\tilde{\bA}_{t,i}} \bigl( \lVert \tilde{\bA}_{t,i}^{-1/2} \bP_{\cS} \bPhi_{\cD_{t}(i)}^{\top} \bPhi_{\cD_{t}(i)} (\bI - \bP_{\cS}) \btheta_{\star} \rVert + \lambda \lVert \btheta_{\star} \rVert_{\tilde{\bA}_{t,i}^{-1}} \bigr) \\
    \leq & \lVert \tilde{\btheta}_{t,i} - \btheta_{\star} \rVert_{\tilde{\bA}_{t,i}} \bigl( \lVert \tilde{\bA}_{t,i}^{-1/2} \bP_{\cS} \bPhi_{\cD_{t}(i)}^{\top} \rVert  \lVert \bPhi_{\cD_{t}(i)} (\bI - \bP_{\cS}) \rVert  \lVert \btheta_{\star} \rVert + \sqrt{\lambda} \lVert \btheta_{\star} \rVert \bigr) \\
    \leq & \lVert \tilde{\btheta}_{t,i} - \btheta_{\star} \rVert_{\tilde{\bA}_{t,i}} \bigl( \lVert \bPhi_{\cD_{t}(i)} (\bI - \bP_{\cS}) \rVert   + \sqrt{\lambda} \bigr) \lVert \btheta_{\star} \rVert
\end{align*}
where the first inequality is due to Cauchy Schwartz, and the last inequality is because $\lVert \tilde{\bA}_{t,i}^{-1/2} \bP_{\cS} \bPhi_{\cD_{t}(i)}^{\top} \rVert = \sqrt{\bPhi_{\cD_{t}(i)} \bP_{\cS} (\bP_{\cS} \bPhi_{\cD_{t}(i)}^{\top}\bPhi_{\cD_{t}(i)} \bP_{\cS} + \lambda \bI)^{-1} \bP_{\cS} \bPhi_{\cD_{t}(i)}^{\top}  } \leq 1$.

\paragraph{Bounding the second term}
By applying Cauchy-Schwartz inequality to the second term, we have
\begin{align*}
    & (\tilde{\btheta}_{t,i} - \btheta_{\star})^{\top} \tilde{\bA}_{t,i} \tilde{\bA}_{t,i}^{-1}\bP_{\cS} \bPhi_{\cD_{t}(i)}^{\top} \mathbf{\eta}_{\cD_{t}(i)} \\
    \leq & \lVert \tilde{\btheta}_{t,i} - \btheta_{\star} \rVert_{\tilde{\bA}_{t,i}} \lVert \tilde{\bA}_{t,i}^{-1/2} \bP_{\cS} \bPhi_{\cD_{t}(i)}^{\top} \mathbf{\eta}_{\cD_{t}(i)} \rVert \\
    = & \lVert \tilde{\btheta}_{t,i} - \btheta_{\star} \rVert_{\tilde{\bA}_{t,i}} \lVert \tilde{\bA}_{t,i}^{-1/2} \bP_{\cS}  \bA_{t,i}^{1/2} \bA_{t,i}^{-1/2} \bPhi_{\cD_{t}(i)}^{\top} \mathbf{\eta}_{\cD_{t}(i)} \rVert \\
    \leq & \lVert \tilde{\btheta}_{t,i} - \btheta_{\star} \rVert_{\tilde{\bA}_{t,i}} \lVert \tilde{\bA}_{t,i}^{-1/2} \bP_{\cS}  \bA_{t,i}^{1/2} \rVert \lVert \bA_{t,i}^{-1/2} \bPhi_{\cD_{t}(i)}^{\top} \mathbf{\eta}_{\cD_{t}(i)} \rVert
\end{align*}
Note that $\bP_{\cS}  \bA_{t,i} \bP_{\cS} = \bP_{\cS}  (\bPhi_{\cD_{t}(i)}^{\top}\bPhi_{\cD_{t}(i)} + \lambda \bI) \bP_{\cS} = \tilde{\bA}_{t,i}  + \lambda (\bP_{\cS} - \bI)$ and $\bP_{\cS} \preceq \bI$, so we have
\begin{align*}
    \lVert \tilde{\bA}_{t,i}^{-1/2} \bP_{\cS}  \bA_{t,i}^{1/2} \rVert & = \sqrt{ \lVert \tilde{\bA}_{t,i}^{-1/2} \bP_{\cS}  \bA_{t,i}^{1/2} \bA_{t,i}^{1/2} \bP_{\cS} \tilde{\bA}_{t,i}^{-1/2}\rVert} \leq \sqrt{\lVert \tilde{\bA}_{t,i}^{-1/2} (  \tilde{\bA}_{t,i} + \lambda (\bP_{\cS} - \bI) ) \tilde{\bA}_{t,i}^{-1/2} \rVert} \\
    & = \sqrt{\lVert \bI + \lambda  \tilde{\bA}_{t,i}^{-1/2} (\bP_{\cS} - \bI) ) \tilde{\bA}_{t,i}^{-1/2} \rVert} \leq \sqrt{1+ \lambda \lVert \tilde{\bA}_{t,i}^{-1} \rVert \lVert \bP_{\cS} - \bI)  \rVert} \\
    & \leq \sqrt{1+ \lambda \cdot \lambda^{-1} \cdot 1 } = \sqrt{2}
\end{align*}
Then using the self-normalized bound derived for Lemma \ref{lem:self_normalized_bound}, the term $\lVert \bA_{t,i}^{-1/2} \bPhi_{\cD_{t}(i)}^{\top} \mathbf{\eta}_{\cD_{t}(i)} \rVert = \lVert \bPhi_{\cD_{t}(i)}^{\top} \mathbf{\eta}_{\cD_{t}(i)} \rVert_{\bA_{t,i}^{-1}}$ can be bounded by
\begin{align*}
    \lVert \bPhi_{\cD_{t}(i)}^{\top} \mathbf{\eta}_{\cD_{t}(i)} \rVert_{\bA_{t,i}^{-1}} & \leq R \sqrt{ 2 \ln(N/\delta) + \ln( \det( \bK_{ \cD_{t}(i), \cD_{t}(i) }/\lambda + \bI) ) } \\
    & \leq R \sqrt{ 2 \ln(N/\delta) + 2 \gamma_{NT} }
\end{align*}
for $\forall t,i$, with probability at least $1-\delta$.
Combining everything finishes the proof.
\end{proof}

Now we are ready to prove Lemma \ref{lem:confidence_ellipsoid_approx} by further bounding the term $\lVert \bPhi_{\cD_{t}(i)} (\bI - \bP_{\cS_{t_{p}}}) \rVert$.
\begin{proof}[Proof of Lemma \ref{lem:confidence_ellipsoid_approx}]
Recall that $\bar{\bS}_{t,i} \in \bR^{|\cD_{t}(i)| \times |\cD_{t}(i)|}$ denotes the diagonal matrix, whose $s$-th diagonal entry equals to $\frac{q_{s}}{\sqrt{\tilde{p}_{s}}}$, where $q_{s}=1$ if $s \in \cS_{t_{p}}$ and $0$ otherwise (note that for $s \notin \cS_{t_{p}}$, we set $\tilde{p}_{s}=1$, so $q_{s}/\tilde{p}_{s}=0$).
Therefore, $\forall s \in \cD_{t}(i) \setminus \cD_{t_{p}}$, $q_{s}=0$, as the dictionary is fixed after $t_{p}$. We can rewrite $\bPhi_{\cD_{t}(i)}^{\top} \bar{\bS}_{t,i}^{\top} \bar{\bS}_{t,i}\bPhi_{\cD_{t}(i)}=\sum_{s \in \cD_{t}(i)} \frac{q_{s}}{\tilde{p}_{s}}\phi_{s} \phi_{s}^{\top}$, where $\phi_{s}:=\phi(\bx_{s})$.
Then by definition of the spectral norm $\lVert \cdot \rVert$, and the properties of the projection matrix $\bP_{\cS_{t_{p}}}$, we have
\begin{align} \label{eq:approx_error}
    & \lVert \bPhi_{\cD_{t}(i)} (\bI - \bP_{\cS_{t_{p}}}) \rVert = \sqrt{\lambda_{\max}\bigl(\bPhi_{\cD_{t}(i)}(\bI - \bP_{\cS_{t_{p}}})^{2}\bPhi_{\cD_{t}(i)}^{\top} \bigr)} = \sqrt{\lambda_{\max}\bigl(\bPhi_{\cD_{t}(i)}(\bI - \bP_{\cS_{t_{p}}})\bPhi_{\cD_{t}(i)}^{\top} \bigr)}.
\end{align}
Moreover, due to Lemma \ref{lem:dictionary_accuracy}, we know $\cS_{t_{p}}$ is $\epsilon_{t,i}$-accurate w.r.t. $\cD_{t}(i)$ for $t \in [t_{p}+1,t_{p+1}]$, where $\epsilon_{t,i}=\bigl(\epsilon+1- \frac{1}{1 + \frac{1+\epsilon}{1-\epsilon} D} \bigr)$, so we have $\bI - \bP_{\cS_{t_{p}}} \preceq \frac{\lambda}{1-\epsilon_{t,i}} (\bPhi_{\cD_{t}(i)}^{\top}\bPhi_{\cD_{t}(i)} + \lambda \bI)^{-1}$ by the property of $\epsilon$-accuracy (Proposition 10 of \cite{calandriello2019gaussian}).
Therefore,
by substituting this back to Eq~\eqref{eq:approx_error}, we have
\begin{align*}
    \lVert \bPhi_{\cD_{t}(i)} (\bI - \bP_{\cS_{t_{p}}}) \rVert 
    \leq \sqrt{\lambda_{\max}\bigl(  \frac{\lambda}{1-\epsilon_{t,i}} \bPhi_{\cD_{t}(i)}(\bPhi_{\cD_{t}(i)}^{\top} \bPhi_{\cD_{t}(i)} + \lambda \bI)^{-1}\bPhi_{\cD_{t}(i)}^{\top} \bigr)}  \leq \sqrt{\frac{\lambda}{-\epsilon+ \frac{1}{1 + \frac{1+\epsilon}{1-\epsilon} D}}}
\end{align*}
which finishes the proof.
\end{proof}

\section{Proof of Theorem \ref{thm:regret_comm_sync}: Regret and Communication Cost of \algtwo{}} \label{sec:proof_appprox_diskernelucb}

\subsection{Regret Analysis}

Consider some time step $t \in [t_{p-1}+1,t_{p}]$, where $p \in [B]$.
Due to Lemma \ref{lem:confidence_ellipsoid_approx}, i.e., the confidence ellipsoid for approximated estimator,
and the fact that $\bx_{t}=\argmax_{\bx \in \cA_{t,i}} \tilde{\mu}_{t-1,i}(\bx) + \alpha_{t-1,i} \tilde{\sigma}_{t-1,i}(\bx)$, we have
\begin{align*}
    f(\bx_{t}^{\star}) & \leq \tilde{\mu}_{t-1,i}(\bx_{t}^{\star}) + \alpha_{t-1,i}\tilde{\sigma}_{t-1,i}(\bx_{t}^{\star}) \leq \tilde{\mu}_{t-1,i}(\bx_{t}) + \alpha_{t-1,i} \tilde{\sigma}_{t-1,i}(\bx_{t}),  \\
    f(\bx_{t}) & \geq \tilde{\mu}_{t-1,i}(\bx_{t}) - \alpha_{t-1,i} \tilde{\sigma}_{t-1,i}(\bx_{t}),
\end{align*}
and thus $r_{t}=f(\bx_{t}^{\star})-f(\bx_{t}) \leq 2 \alpha_{t-1,i} \tilde{\sigma}_{t-1,i}(\bx_{t})$, where 
\begin{align*}
    \alpha_{t-1,i}=\Bigg( \frac{1}{\sqrt{-\epsilon + \frac{1}{1+\frac{1+\epsilon}{1-\epsilon}D}}}  + 1  \Bigg) \sqrt{\lambda} \lVert \btheta_{\star} \rVert + R \sqrt{4 \ln{N/\delta}+ 2\ln{\det((1+\lambda) \bI + \bK_{\cD_{t-1}(i), \cD_{t-1}(i)})}}.
\end{align*}
Note that, different from Appendix \ref{sec:proof_regret_diskernelucb}, the $\alpha_{t-1,i}$ term now depends on the threshold $D$ and accuracy constant $\epsilon$, as a result of the approximation error. As we will see in the following paragraphs, their values need to be set properly in order to bound $\alpha_{t-1,i}$.

We begin the regret analysis of \algtwo{} with the same decomposition of good and bad epochs as in Appendix \ref{subsec:regret_proof_diskernelucb}, i.e., we call the $p$-th epoch a good epoch if $\ln(\frac{\det(\bI + \lambda^{-1}\bK_{[t_{p}],[t_{p}]})}{\det(\bI + \lambda^{-1}\bK_{[t_{p-1}],[t_{p-1}]})}) \leq 1$, otherwise it is a bad epoch. Moreover, due to the pigeon-hold principle, there can be at most $2 \gamma_{NT}$ bad epochs.

As we will show in the following paragraphs, using Lemma \ref{lem:rls}, we can obtain a similar bound for the cumulative regret in good epochs as that in Appendix \ref{subsec:regret_proof_diskernelucb}, but with additional dependence on $D$ and $\epsilon$. The proof mainly differs in the bad epochs, where we need to use the event-trigger in Eq~\eqref{eq:sync_event} to bound the cumulative regret in each bad epoch. Compared with Eq~\eqref{eq:sync_event_exact}, Eq~\eqref{eq:sync_event} does not contain the number of local updates on each client since last synchronization., and as mentioned in Section \ref{subsec:theoretical_analysis}, this introduces a $\sqrt{T}$ factor in the regret bound for bad epochs in place of the $\sqrt{\gamma_{NT}}$ term in Appendix \ref{subsec:regret_proof_diskernelucb}.

\paragraph{Cumulative Regret in Good Epochs}
Let's first consider some time step $t$ in a good epoch $p$, i.e., $t \in [t_{p-1}+1, t_{p}]$, and we have the following bound on the instantaneous regret
\begin{align*}
    r_{t} & \leq 2 \alpha_{t-1,i} \tilde{\sigma}_{t-1,i}(\bx_{t}) \leq 2 \alpha_{t-1,i} \tilde{\sigma}_{t_{p-1},i}(\bx_{t}) \leq 2 \alpha_{t-1,i} \frac{1+\epsilon}{1-\epsilon} \sigma_{t_{p-1},i}(\bx_{t}) \\
    & = 2 \alpha_{t-1,i} \frac{1+\epsilon}{1-\epsilon} \sqrt{ \phi_{t}^{\top} A_{t_{p-1}}^{-1} \phi_{t}} \leq 2 \alpha_{t-1,i} \frac{1+\epsilon}{1-\epsilon} \sqrt{ \phi_{t}^{\top} A_{t-1}^{-1} \phi_{t}} \sqrt{\frac{\det(\bI + \lambda^{-1}\bK_{[t-1],[t-1]})}{\det(\bI + \lambda^{-1}\bK_{[t_{p-1}],[t_{p-1}]})}} \\
    & \leq 2 \sqrt{e} \frac{1+\epsilon}{1-\epsilon} \alpha_{t-1,i} \sqrt{ \phi_{t}^{\top} A_{t-1}^{-1} \phi_{t}}
\end{align*}
where the second inequality is because the (approximated) variance is non-decreasing, the third inequality is due to Lemma \ref{lem:rls}, the forth is due to Lemma \ref{lem:quadratic_det_inequality_infinite}, and the last is because in a good epoch, 
we have $\frac{\det(\bI + \lambda^{-1}\bK_{[t-1],[t-1]})}{\det(\bI + \lambda^{-1}\bK_{[t_{p-1}],[t_{p-1}]})} \leq \frac{\det(\bI + \lambda^{-1}\bK_{[t_{p}],[t_{p}]})}{\det(\bI + \lambda^{-1}\bK_{[t_{p-1}],[t_{p-1}]})} \leq e$ for $t \in [t_{p-1}+1, t_{p}]$.

Therefore, the cumulative regret incurred in all good epochs, denoted by $R_{good}$, is upper bounded by
\begin{align*}
    R_{good} & \leq 2 \sqrt{e} \frac{1+\epsilon}{1-\epsilon} \sum_{t=1}^{NT} \alpha_{t-1,i} \sqrt{ \phi_{t}^{\top} A_{t-1}^{-1} \phi_{t}} \leq 2 \sqrt{e} \frac{1+\epsilon}{1-\epsilon} \alpha_{NT} \sqrt{NT \cdot \sum_{t=1}^{NT} \phi_{t}^{\top} A_{t-1}^{-1} \phi_{t} }  \\
    & \leq 2 \sqrt{e} \frac{1+\epsilon}{1-\epsilon} \alpha_{NT} \sqrt{NT \cdot 2\gamma_{NT} } 
\end{align*}
where $\alpha_{NT}:=\Bigg( \frac{1}{\sqrt{-\epsilon + \frac{1}{1+\frac{1+\epsilon}{1-\epsilon}D}}}  + 1  \Bigg) \sqrt{\lambda} \lVert \btheta_{\star} \rVert + R \sqrt{4 \ln{N/\delta}+ 2\ln{\det((1+\lambda) \bI + \bK_{[NT],[NT]})}}$.
\paragraph{Cumulative Regret in Bad Epochs}


The cumulative regret incurred in this bad epoch is
\scriptsize
\begin{align*}
    & \sum_{p=1}^{B} \mathbbm{1}\{\ln(\frac{\det(\bI + \lambda^{-1}\bK_{[t_{p}],[t_{p}]})}{\det(\bI + \lambda^{-1}\bK_{[t_{p-1}],[t_{p-1}]})}) > 1\} \sum_{t=t_{p-1}+1}^{t_{p}} r_{t} \\
    & \leq 2 \sum_{p=1}^{B} \mathbbm{1}\{\ln(\frac{\det(\bI + \lambda^{-1}\bK_{[t_{p}],[t_{p}]})}{\det(\bI + \lambda^{-1}\bK_{[t_{p-1}],[t_{p-1}]})}) > 1\} \sum_{t=t_{p-1}+1}^{t_{p}} \alpha_{t-1,i} \tilde{\sigma}_{t-1,i}(\bx_{t}) \\
    & \leq 2 \alpha_{NT} \sum_{p=1}^{B} \mathbbm{1}\{\ln(\frac{\det(\bI + \lambda^{-1}\bK_{[t_{p}],[t_{p}]})}{\det(\bI + \lambda^{-1}\bK_{[t_{p-1}],[t_{p-1}]})}) > 1\} \sum_{i=1}^{N} \sum_{t \in \cN_{t_{p}}(i) \setminus \cN_{t_{p-1}}(i) }\tilde{\sigma}_{t-1,i}(\bx_{t}) \\
    & \leq 2 \alpha_{NT} \sum_{p=1}^{B} \mathbbm{1}\{\ln(\frac{\det(\bI + \lambda^{-1}\bK_{[t_{p}],[t_{p}]})}{\det(\bI + \lambda^{-1}\bK_{[t_{p-1}],[t_{p-1}]})}) > 1\} \sum_{i=1}^{N} \sqrt{ (|\cN_{t_{p}}(i)|- |\cN_{t_{p-1}}(i)| ) \sum_{t \in \cN_{t_{p}}(i) \setminus \cN_{t_{p-1}}(i) }\tilde{\sigma}^{2}_{t-1,i}(\bx_{t})  } \\
    & \leq 2 \alpha_{NT} \sqrt{D} \sum_{p=1}^{B} \mathbbm{1}\{\ln(\frac{\det(\bI + \lambda^{-1}\bK_{[t_{p}],[t_{p}]})}{\det(\bI + \lambda^{-1}\bK_{[t_{p-1}],[t_{p-1}]})}) > 1\}  \sum_{i=1}^{N} \sqrt{ (|\cN_{t_{p}}(i)|- |\cN_{t_{p-1}}(i)| )  } \\
    & \leq 2 \alpha_{NT} \sqrt{D} \sum_{p=1}^{B} \mathbbm{1}\{\ln(\frac{\det(\bI + \lambda^{-1}\bK_{[t_{p}],[t_{p}]})}{\det(\bI + \lambda^{-1}\bK_{[t_{p-1}],[t_{p-1}]})}) > 1\} \sum_{i=1}^{N} \sqrt{ \frac{t_{p}-t_{p-1}}{N} } \\
    & \leq 2 \alpha_{NT} \sqrt{DN} \sum_{p=1}^{B} \mathbbm{1}\{\ln(\frac{\det(\bI + \lambda^{-1}\bK_{[t_{p}],[t_{p}]})}{\det(\bI + \lambda^{-1}\bK_{[t_{p-1}],[t_{p-1}]})}) > 1\}  \sqrt{t_{p}-t_{p-1}}  \\
    & \leq 2 \alpha_{NT} \sqrt{DN} \sqrt{ \sum_{p=1}^{B} \mathbbm{1}\{\ln(\frac{\det(\bI + \lambda^{-1}\bK_{[t_{p}],[t_{p}]})}{\det(\bI + \lambda^{-1}\bK_{[t_{p-1}],[t_{p-1}]})}) > 1\} (t_{p}-t_{p-1}) \cdot  \sum_{p=1}^{B} \mathbbm{1}\{\ln(\frac{\det(\bI + \lambda^{-1}\bK_{[t_{p}],[t_{p}]})}{\det(\bI + \lambda^{-1}\bK_{[t_{p-1}],[t_{p-1}]})}) > 1\}  } \\
    & \leq 2 \alpha_{NT} \sqrt{DN} \sqrt{2 N T \gamma_{NT}}
\end{align*}
\normalsize
where the third inequality is due to the Cauchy-Schwartz inequality, the forth is due to our event-trigger in Eq~\eqref{eq:sync_event}, the fifth is due to our assumption that clients interact with the environment in a round-robin manner, the sixth is due to the Cauchy-Schwartz inequality again, and the last is due to the fact that there can be at most $2\gamma_{NT}$ bad epochs.

Combining cumulative regret incurred during both good and bad epochs, we have
\begin{align*}
    R_{NT} \leq R_{good} + R_{bad} \leq 2 \sqrt{e} \frac{1+\epsilon}{1-\epsilon} \alpha_{NT} \sqrt{NT \cdot 2\gamma_{NT} }  +  2 \alpha_{NT} \sqrt{DN} \sqrt{2 N T \gamma_{NT}} 
\end{align*}

\subsection{Communication Cost Analysis}
Consider some epoch $p$. We know that for the client $i$ who triggers the global synchronization, we have
\begin{align*}
    \frac{1+\epsilon}{1-\epsilon}\sum_{s = t_{p-1}+1}^{t_{p}} \sigma^{2}_{t_{p-1}}(\bx_{s}) \geq 
    \sum_{s = t_{p-1}+1}^{t_{p}} \tilde{\sigma}^{2}_{t_{p-1}}(\bx_{s}) \geq \sum_{s \in \cD_{t_{p}(i)} \setminus \cD_{t_{p-1}(i)}} \tilde{\sigma}^{2}_{t_{p-1}}(\bx_{s}) \geq D
\end{align*}
Then by summing over $B$ epochs, we have
\begin{align*}
    B D < \frac{1+\epsilon}{1-\epsilon} \sum_{p=1}^{B} \sum_{s = t_{p-1}+1}^{t_{p}} \sigma^{2}_{t_{p-1}}(\bx_{s}) \leq \frac{1+\epsilon}{1-\epsilon} \sum_{p=1}^{B} \sum_{s = t_{p-1}+1}^{t_{p}} \sigma^{2}_{s-1}(\bx_{s}) \frac{\sigma^{2}_{t_{p-1}}(\bx_{s})}{\sigma^{2}_{s-1}(\bx_{s})}.
\end{align*}
Now we need to bound the ratio $\frac{\sigma^{2}_{t_{p-1}}(\bx_{s})}{\sigma^{2}_{s-1}(\bx_{s})}$ for $s \in [t_{p-1}+1,t_{p}]$.
\begin{align*}
    &  \frac{\sigma^{2}_{t_{p-1}}(\bx_{s})}{\sigma^{2}_{s-1}(\bx_{s})}  \leq \Big[ 1 + \sum_{\tau = t_{p-1}+1}^{s} \sigma^{2}_{t_{p-1}}(\bx_{\tau}) \Big]  \leq \Big[ 1 + \frac{1+\epsilon}{1-\epsilon} \sum_{\tau = t_{p-1}+1}^{s} \tilde{\sigma}^{2}_{t_{p-1}}(\bx_{\tau}) \Big] 
\end{align*}
Note that for the client who triggers the global synchronization, we have $\sum_{s \in \cD_{t_{p}-1}(i) \setminus \cD_{t_{p-1}}(i)} \tilde{\sigma}^{2}_{t_{p-1}}(\bx_{s}) < D$, i.e., one time step before it triggers the synchronization at time $t_{p}$.
Due to the fact that the (approximated) posterior variance cannot exceed $L^{2}/\lambda$, we have $\sum_{s \in \cD_{t_{p}}(i) \setminus \cD_{t_{p-1}}(i)} \tilde{\sigma}^{2}_{t_{p-1}}(\bx_{s}) < D+ L^{2}/\lambda$.
For the other $N-1$ clients, we have $\sum_{s \in \cD_{t_{p}}(i) \setminus \cD_{t_{p-1}}(i)} \tilde{\sigma}^{2}_{t_{p-1}}(\bx_{s}) < D$.
Summing them together, we have
\begin{align*}
    \sum_{s = t_{p-1}+1}^{t_{p}} \tilde{\sigma}^{2}_{t_{p-1}}(\bx_{s}) < (N D + L^{2}/\lambda)
\end{align*}
for the $p$-th epoch. By substituting this back, we have
\begin{align*}
    \frac{\sigma^{2}_{t_{p-1}}(\bx_{s})}{\sigma^{2}_{s-1}(\bx_{s})} \leq \Big[ 1 + \frac{1+\epsilon}{1-\epsilon} (N D + L^{2}/\lambda) \Big]
\end{align*}
Therefore,
\begin{align*}
    B D & < \frac{1+\epsilon}{1-\epsilon}\Big[ 1 + \frac{1+\epsilon}{1-\epsilon} (N D + L^{2}/\lambda) \Big] \sum_{p=1}^{B} \sum_{s = t_{p-1}+1}^{t_{p}} \sigma^{2}_{s-1}(\bx_{s})  \\
    & \leq \frac{1+\epsilon}{1-\epsilon}\Big[ 1 + \frac{1+\epsilon}{1-\epsilon} (N D + L^{2}/\lambda) \Big] 2 \gamma_{NT}
\end{align*}
and thus the total number of epochs $B < \frac{1+\epsilon}{1-\epsilon}[ \frac{1}{D} + \frac{1+\epsilon}{1-\epsilon} (N + L^{2}/(\lambda D)) ] 2 \gamma_{NT}$.

By setting $D=\frac{1}{N}$, we have
\begin{align*}
    \alpha_{NT} & =\Bigg( \frac{1}{\sqrt{-\epsilon + \frac{1}{1+\frac{1+\epsilon}{1-\epsilon} \frac{1}{N}}}}  + 1  \Bigg) \sqrt{\lambda} \lVert \btheta_{\star} \rVert + R \sqrt{4 \ln{N/\delta}+ 2\ln{\det((1+\lambda) \bI + \bK_{[NT],[NT]})}} \\
    & \leq \Bigg( \frac{1}{\sqrt{-\epsilon + \frac{1}{1+\frac{1+\epsilon}{1-\epsilon} }}}  + 1  \Bigg) \sqrt{\lambda} \lVert \btheta_{\star} \rVert + R \sqrt{4 \ln{N/\delta}+ 2\ln{\det((1+\lambda) \bI + \bK_{[NT],[NT]})}}
\end{align*}
because $N \geq 1$. Moreover, to ensure $-\epsilon + \frac{1}{1+\frac{1+\epsilon}{1-\epsilon} } > 0$, we need to set the constant $\epsilon < 1/3$. Therefore,
\begin{align*}
    R_{NT}=O\Big( \sqrt{NT}  ( \lVert \theta_{\star} \rVert \sqrt{\gamma_{NT}} + \gamma_{NT}) \Big)
\end{align*}
and the total number of global synchronizations $B=O(N\gamma_{NT})$. Since for each global synchronization, the communication cost is $O(N\gamma_{NT}^{2})$, we have
\begin{align*}
    C_{NT} = O \Big( N^{2} \gamma_{NT}^{3} \Big)
\end{align*}

\section{Experiment Setup}

\paragraph{Synthetic dataset}
We simulated the distributed bandit setting defined in Section \ref{subsec:problem_formulation}, with $d=20,T=100, N=100$ ($NT=10^{4}$ interactions in total). In each round $l \in [T]$, each client $i \in [N]$ (denote $t=N(l-1)+i$) selects an arm from candidate set $\cA_{t}$, 
where $\cA_{t}$ is uniformly sampled from a $\ell_2$ unit ball, with $|\cA_{t}|=20$.
Then the corresponding reward is generated using one of the following reward functions:
\begin{align*}
& f_{1}(\bx) = \cos(3 \bx^{\top} \btheta_{\star}) \\
& f_{2}(\bx) = (\bx^{\top} \btheta_{\star})^{3} - 3(\bx^{\top} \btheta_{\star})^{2} - (\bx^{\top} \btheta_{\star}) + 3
\end{align*}
where the parameter $\btheta_{\star}$ is uniformly sampled from a $\ell_2$ unit ball.


\paragraph{UCI Datasets}
To evaluate \algtwo{}'s performance in a more challenging and practical scenario, we performed experiments using real-world datasets: MagicTelescope, Mushroom and Shuttle from the UCI Machine Learning Repository \citep{Dua:2019}. To convert them to contextual bandit problems, we pre-processed these datasets following the steps in \citep{filippi2010parametric}. In particular, we partitioned the dataset in to $20$ clusters using k-means, and used the centroid of each cluster as the context vector for the arm and the averaged response variable as mean reward (the response variable is binarized by
associating one class as $1$, and all the others as $0$). Then we simulated the distributed bandit learning problem in Section \ref{subsec:problem_formulation} with $|\cA_{t}|=20$, $T=100$ and $N=100$ ($NT=10^{4}$ interactions in total).

\paragraph{MovieLens and Yelp dataset}
Yelp dataset, which is released by the Yelp dataset challenge, consists of 4.7 million rating entries for 157 thousand restaurants by 1.18 million users. MovieLens is a dataset consisting of 25 million ratings between 160 thousand users and 60 thousand movies \citep{harper2015movielens}. Following the pre-processing steps in \citep{ban2021ee}, we built the rating matrix by choosing the top 2000 users and top 10000 restaurants/movies and used singular-value decomposition (SVD) to extract a 10-dimension feature vector for each user and restaurant/movie. We treated rating greater than $2$ as positive. 
We simulated the distributed bandit learning problem in Section \ref{subsec:problem_formulation} with $T=100$ and $N=100$ ($NT=10^{4}$ interactions in total).
In each time step, the candidate set $\cA_{t}$ (with $|\cA_{t}|=20$) is constructed by sampling an arm with reward $1$ and nineteen arms with reward $0$ from the arm pool, and the concatenation of user and restaurant/movie feature vector is used as the context vector for the arm (thus $d=20$).

\section{Lower Bound for Distributed Kernelized Contextual Bandits}

First, we need the following two lemmas
\begin{lemma}[Theorem 1 of \cite{valko2013finite}] \label{lem:KernelUCB_upperbound}
There exists a constant $C > 0$, such that for any instance of kernelized bandit with $L=S=R= 1$, the expected cumulative regret for KernelUCB algorithm is upper bounded by $\bbE[R_{T}] \leq C\sqrt{T\gamma_{T}}$, where the maximum information gain $\gamma_{T}=O\bigl((\ln(T))^{d+1}\bigr)$ for Squared Exponential kernels.
\end{lemma}

\begin{lemma}[Theorem 2 of \cite{scarlett2017lower}] \label{lem:squaredexponential_lowerbound}
There always exists a set of hard-to-learn instances of kernelized bandit with $L=S=R= 1$, such that for any algorithm, for a uniformly random instance in the set, the expected cumulative regret $\bbE[R_{T}] \geq c \sqrt{T (\ln(T))^{d/2}}$ for Squared Exponential kernels,
with some constant $c$.
\end{lemma}

Then we follow a similar procedure as the proof for Theorem 2 of \cite{wang2019distributed} and Theorem 5.3 of \cite{he2022simple}, to prove the following lower bound results for distributed kernelized bandit with Squared Exponential kernels.

\begin{theorem}\label{thm:diskernel_lowerbound}
For any distributed kernelized bandit algorithm with expected communication cost less than $O(\frac{N}{(\ln(T))^{0.25 d + 0.5}})$, there exists a kernelized bandit instance with Squared Exponential kernel, and $L=S=R=1$, such that the expected cumulative regret for this algorithm is at least $\Omega(N \sqrt{T(\ln(T))^{d/2}})$.
\end{theorem}

\begin{proof}[Proof of Theorem \ref{thm:diskernel_lowerbound}]
Here we consider kernelized bandit with Squared Exponential kernels.
The proof relies on the construction of a auxiliary algorithm, denoted by \textbf{AuxAlg}, based on the original distributed kernelized bandit algorithm, denoted by \textbf{DisKernelAlg}, as shown below.
For each agent $i \in [N]$, \textbf{AuxAlg} performs \textbf{DisKernelAlg}, until any communication happens between client $i$ and the server, in which case, \textbf{AuxAlg} switches to the single-agent optimal algorithm, i.e., the KernelUCB algorithm that attains the rate in Lemma \ref{lem:KernelUCB_upperbound}. Therefore, \textbf{AuxAlg} is a single-agent bandit algorithm, and the lower bound in Lemma \ref{lem:squaredexponential_lowerbound} applies: the cumulative regret that \textbf{AuxAlg} incurs for some agent $i\in[N]$ is lower bounded by
\begin{align*}
    \bbE[R_{\textbf{AuxAlg},i}] \geq c \sqrt{T (\ln(T))^{d/2}},
\end{align*}
and by summing over all $N$ clients, we have
\begin{align*}
    \bbE[R_{\textbf{AuxAlg}}] = \sum_{i=1}^{N}\bbE[R_{\textbf{AuxAlg},i}] \geq c N\sqrt{T (\ln(T))^{d/2}}.
\end{align*}

For each client $i \in [N]$, denote the probability that client $i$ will communicate with the server as $p_{i}$, and $p:=\sum_{i=1}^{N} p_{i}$. Note that before the communication, the cumulative regret incurred by \textbf{AuxAlg} is the same as \textbf{DisKernelAlg}, and after the communication happens, the regret incurred by \textbf{AuxAlg} is the same as KernelUCB, whose upper bound is given in Lemma \ref{lem:KernelUCB_upperbound}. Therefore, the cumulative regret that \textbf{AuxAlg} incurs for client $i$ can be upper bounded by
\begin{align*}
        \bbE[R_{\textbf{AuxAlg},i}] \leq \bbE[R_{\textbf{DisKernelAlg},i}] + p_{i} C\sqrt{T(\ln(T))^{d+1}},
\end{align*}
and by summing over $N$ clients, we have
\begin{align*}
    \bbE[R_{\textbf{AuxAlg}}] & = \sum_{i=1}^{N}\bbE[R_{\textbf{AuxAlg},i}]  \\
    & \leq \sum_{i=1}^{N} \bbE[R_{\textbf{DisKernelAlg},i}] + (\sum_{i=1}^{N} p_{i}) C\sqrt{T (\ln(T))^{d+1}}  \\
    & = \bbE[R_{\textbf{DisKernelAlg}}] + p C\sqrt{T (\ln(T))^{d+1}}.
\end{align*}
Combining the upper and lower bounds for $\bbE[R_{\textbf{AuxAlg}}]$, we have
\begin{align*}
    \bbE[R_{\textbf{DisKernelAlg}}] \geq c N\sqrt{T (\ln(T))^{d/2}} - p C\sqrt{T (\ln(T))^{d+1}}.
\end{align*}
Therefore, for any \textbf{DisKernelAlg} with number of communications $p\leq N \frac{c}{2C (\ln(T))^{0.25 d + 0.5}}=O(\frac{N}{(\ln(T))^{0.25 d + 0.5}})$, we have
\begin{align*}
    \bbE[R_{\textbf{DisKernelAlg}}] \geq \frac{c}{2} N\sqrt{T (\ln(T))^{d/2}} = \Omega(N \sqrt{T(\ln(T))^{d/2}}).
\end{align*}

\end{proof}

\end{document}